%% file: arxiv.tex
\documentclass[12pt]{article}
\usepackage[margin=1in]{geometry}
\input{preamble}

\usepackage{parskip}
\date{}

\def\arxiv{1}

\begin{document}

\title{Private Protocols for $U$-Statistics in the Local Model and Beyond}
\author{James
Bell\thanks{The Alan Turing Institute. {\tt jbell@turing.ac.uk}. Work supported by The Alan Turing Institute under the EPSRC
grant EP/N510129/1, and the UK Government’s Defence \& Security Programme in support of the Alan Turing Institute.}~~~~~Aurélien Bellet\thanks{INRIA. {\tt aurelien.bellet@inria.fr}. Work
supported by grants
ANR-16-CE23-0016-01 and ANR-18-CE23-0018-03, by the European Union's Horizon
2020 Research and Innovation Program under Grant Agreement No. 825081 COMPRISE
and by a grant from CPER Nord-Pas de Calais/FEDER DATA Advanced data science
and technologies 2015-2020. A. B. thanks Jan Ramon for useful
discussions.}~~~~~Adri{\`a} Gasc{\'o}n \thanks{Google. {\tt
adriagascon@gmail.com}. Work part done when A. G. was at The Alan Turing
Institute and Warwick University, and supported by The Alan Turing Institute
under the EPSRC grant EP/N510129/1, and the UK Government’s Defence \&
Security Programme in support of the Alan Turing Institute.}~~~~~Tejas
Kulkarni 
\thanks{Aalto University, Helsinki, Finland. T. K. thanks Graham Cormode for
arranging a trip to Inria, Lille. His visit was supported by Marie Curie Grant 618202.}}

\maketitle

\input{subfiles/abstract}

\input{subfiles/intro}
\input{subfiles/background}
\input{subfiles/generic_ldp}
\input{subfiles/auc}
\input{subfiles/generic_2pc}
\input{subfiles/exp}
\input{subfiles/conclu}

\bibliographystyle{apalike}
\bibliography{papers}

\appendix

\input{subfiles/appendix}

\end{document}

%% file: preamble.tex
\usepackage[utf8]{inputenc} 
\usepackage[T1]{fontenc}    
\usepackage{hyperref}       
\usepackage{url}            
\usepackage{booktabs}       
\usepackage{amsfonts}       
\usepackage{nicefrac}       
\usepackage{microtype}      
\usepackage{amssymb}      
\usepackage{color}
\usepackage{multirow}
\usepackage[noend]{algorithmic}
\usepackage[algoruled,nofillcomment,algo2e]{algorithm2e}
\usepackage{algorithm}
\usepackage{natbib}

\usepackage{enumitem}
\usepackage{amsmath}
\usepackage{qtree}

\usepackage{booktabs}

\newcommand{\var}{\mathrm{Var}} 

\usepackage{blkarray}
\usepackage{graphicx}
\usepackage{graphicx}
\usepackage{subfigure}

\DeclareMathOperator*{\E}{\mathbb{E}}
\DeclareMathOperator{\sgn}{sign}
\usepackage{amsthm}
\newtheorem{definition}{Definition}
\newtheorem{lemma}{Lemma}
\newtheorem{theorem}{Theorem}
\newtheorem{remark}{Remark}
\newtheorem{corollary}{Corollary}

\usepackage[textsize=footnotesize,color=green!40]{todonotes}

\newcommand{\RRR}{\mathcal{R}}
\newcommand{\lr}{\mathcal{R}}
\newcommand{\EE}{\mathbb{E}}
\newcommand{\PP}{\mathbb{P}}

\newcommand{\mse}{\texttt{MSE}}
\newcommand{\interval}{\mathcal{I}}
\newcommand{\recursed}{R}
\newcommand{\discarded}{D}
\newcommand{\activated}{A}

\newcommand{\dataset}{\mathcal{S}}
\newcommand{\auc}{\texttt{AUC}}
\newcommand{\uauc}{\texttt{UAUC}}
\newcommand{\aucest}{\widehat{\texttt{AUC}}}
\newcommand{\uaucest}{\widehat{\texttt{UAUC}}}
\newcommand{\fgini}{\texttt{fgini}}
\newcommand{\fauc}{\texttt{fauc}}

%% file: subfiles/abstract.tex

\begin{abstract}
  In this paper, we study the problem of computing $U$-statistics of degree $2$,
  i.e., quantities that come in the form of averages over pairs of data points,
  in the local model of differential privacy (LDP). The class
  of $U$-statistics covers many statistical estimates of interest, including
  Gini mean difference, Kendall's tau coefficient
  and Area under the ROC Curve (AUC), as well as empirical risk measures for
  machine learning problems such as ranking, clustering and metric learning.
  We first introduce an LDP protocol based on quantizing the data into
  bins and applying randomized response, which guarantees an $\epsilon$-LDP
  estimate with a Mean Squared Error (MSE) of $O(1/\sqrt{n}\epsilon)$ under
  regularity assumptions on the $U$-statistic or the data distribution.
  We then propose a specialized protocol for AUC based on a novel use of
  hierarchical histograms that achieves MSE of $O(\alpha^3/n\epsilon^2)$ for
  arbitrary data distribution.
  We also show that 2-party secure computation allows to design a protocol
  with MSE of $O(1/n\epsilon^2)$, without
  any assumption on the kernel function or data distribution and with total communication linear in the number of users $n$.
  Finally, we evaluate the performance of our protocols through experiments on
  synthetic and real datasets.
\end{abstract}

%% file: subfiles/intro.tex

\if\arxiv1
\section{Introduction}
\else
\section{INTRODUCTION}
\fi

The problem of collecting aggregate statistics from a set of $n$ users in a way
that individual contributions remain private even from the
data analysts has recently attracted a lot of interest. In the
popular \emph{local model} of differential privacy (LDP) \citep{d:13,ksv:14},
users apply a local randomizer to their private input before
sending it to an untrusted aggregator. In this context, most work has focused
on computing quantities that are separable across individual users,
such as sums and histograms \citep[see][and
references therein]
{sb:15,Wang:Blocki:Li:Jha:17,KCS:2019,CKS:18,HeavyHitters:17}.

In this paper, we study the problem of privately computing
$U$-statistics of degree $2$, which generalize sample mean statistics to 
\emph{averages over pairs of data points}. Let
$x_1,\dots,x_n$ be a set of $n$ data points drawn i.i.d. from an
unknown ditribution $\mu$ over a (discrete or continuous)
domain $\mathcal{X}$. The $U$-statistic of degree $2$ with kernel $f$, given
by $U_{f,n}=\frac{2}{n(n-1)}\sum_{i<j}f(x_i,x_j)$, is an unbiased estimate
of
$U_f=\EE_{x,x'\sim \mu}[f(x,x')]$ with
minimum variance \citep{Hoeffding48}.
The class of $U$-statistics covers many statistical estimates of interest,
including sample variance, Gini mean difference, Kendall's tau
coefficient, Wilcoxon Mann-Whitney hypothesis test and Area under the ROC
Curve (AUC) \citep{Lee90,wilcoxon,meanNN}. They are also commonly used as empirical risk measures for machine learning problems such as ranking, clustering and metric learning \citep{pairwise-losses,CBC2016}.

Interestingly, private estimation of $U$-statistics in the LDP model for
arbitrary kernel functions $f$ and data distributions $\mu$ cannot be
straightforwardly addressed by resorting to standard local randomizers such as
the Laplace mechanism or randomized response.
Indeed, one cannot apply the local randomizer
to the terms of the sum based on the sensitivity of $f$ (as each term is
shared across two users), and perturbing the inputs themselves can lead to
large errors when passed through the (potentially discontinuous) function $f$.

In this work, we design and analyze several
protocols for computing $U$-statistics with privacy and utility
guarantees. More precisely:
\begin{enumerate}\setlength\itemsep{.0em}
\item We introduce a
generic LDP protocol based on quantizing the data into $k$ bins and applying
$k$-ary randomized response. We show that under an assumption on either the
kernel function $f$ or the data distribution $\mu$, the aggregator can
construct an $\epsilon$-LDP estimate of $U_{f,n}$ with a Mean Squared Error 
(MSE) of $O(1/\sqrt{n}\epsilon)$.
\item For the case of the AUC on a domain of
size $2^\alpha$, whose kernel does not satisfy the regularity assumption
required by our previous protocol, we design a specialized protocol based
on hierarchical histograms that achieves MSE $O(\alpha^2\log
(1/\delta)/n\epsilon^2)$ under $(\epsilon,\delta)$-LDP and $O
(\alpha^3/n\epsilon^2)$ under $\epsilon$-LDP, for arbitrary data distribution.
\item Under a slight
relaxation of the local model in which we allow pairs of users $i$ and $j$ to
compute a randomized version of $f(x_i,x_j)$ with 2-party secure computation,
we show that we can design a protocol with MSE of $O(1/n\epsilon^2)$, without
any assumption on the kernel function or data distribution and with constant communication for each of the $n$ users.
\item To evaluate the practical performance of the proposed protocols, we
present some experiments on synthetic and real datasets for the task of
computing AUC and Kendall's tau coefficient.
\end{enumerate}

The paper is organized as follows. Section~\ref{sec:background} gives some
background on $U$-statistics and local differential privacy. In Section~\ref{sec:generic_ldp} we present a generic LDP protocol based on randomizing
quantized inputs. Section~\ref{sec:auc} introduces a specialized LDP protocol
for computing the Area under the ROC Curve (AUC). In Section~\ref{sec:generic_2pc}, we introduce a generic protocol which operates in a
slightly relaxed version of the LDP model where users can run secure 2-party
computation. We present some numerical experiments in Section~\ref{sec:experiments}, and conclude with a discussion of our results and
future work in Section~\ref{sec:conclu}.

%% file: subfiles/background.tex

\if\arxiv1
\section{Background}
\else
\section{BACKGROUND}
\fi
\label{sec:background}

In this section, we introduce some background on $U$-statistics and local
differential privacy.

\subsection{$U$-Statistics}

\subsubsection{Definition and Properties}

Let $\mu$ be an (unknown) distribution over an input space $\mathcal{X}$ and
$f: \mathcal{X}^{2} \rightarrow \mathbb{R}$ be a pairwise function (assumed
to be symmetric for simplicity) referred to as the \emph{kernel}.
Given a
sample $\dataset=\{x_i\}_{i=1}^n$ of $n$ observations drawn from $\mu$, we
are interested in estimating the following population quantity:
\begin{equation}
\label{eq:population}
U_f=\mathbb{E}_{X_1,X_2\sim\mu}[f(X_{1},X_{2})].
\end{equation}
\begin{definition}[\citealp{Hoeffding48}]
The $U$-statistic of degree $2$ with kernel $f$ is given by
\begin{equation}
\label{eq:ustat}
U_{f,n}= \textstyle\frac{2}{n(n-1)}\sum_{i<j} f(x_i,x_j).
\end{equation}
\end{definition}

$U_{f,n}$ is an unbiased estimate of $U_f$.
Denoting by $\zeta_1=\var(f(x_1,X_2)\mid x_1))$ and $\zeta_2=\var(f
(X_1,X_2)$, its variance is given by 
\citep{Hoeffding48,Lee90}:
\begin{equation}
\label{eq:ustat-variance}
\var(U_{f,n}) = \textstyle\frac{2}{n(n-1)}(2(n-2)\zeta_1+\zeta_2).
\end{equation}

The above variance is of $O(1/n)$ and is optimal among all unbiased estimators
of $U_f$ that can be computed from $\dataset$.
This incurs a complex dependence structure, as each
data point appears in $n-1$ pairs.
The statistical behavior of $U$-statistics can be investigated using
linearization techniques \citep{Hoeffding48} and decoupling methods 
\citep{PenaGine99}, which provide tools to reduce their analysis to that of
standard i.i.d. averages. One may refer to \citep{Lee90} for asymptotic theory
of $U$-statistics, to \citep{VVaart} (Chapter 12 therein) and
\citep{PenaGine99} for nonasymptotic results, and to \citep{CLV08,CBC2016} for
an account of $U$-statistics in the context of
machine learning and empirical risk minimization.

\subsubsection{Motivating Examples}
\label{sec:ustat-ex}

$U$-statistics are commonly used as point estimators of various global
properties of distributions, as well as in statistical hypothesis
testing \citep{Lee90,wilcoxon,meanNN}.
They also come up as empirical risk measures in machine learning
problems with pairwise loss functions such as bipartite ranking, metric
learning and clustering.
Below, we give some concrete examples of $U$-statistics of broad interest to
motivate our private protocols.



\textbf{Gini mean difference.} This is a classic measure of dispersion
which is often seen as more informative than the variance for
some distributions \citep{Gini}.
Letting $\mathcal{X}\subset\mathbb{R}$, it is defined as
\begin{equation}
\label{eq:gini}
G=\textstyle\frac{2}{n(n-1)} \sum_{i <j } |x_i - x_j|,
\end{equation}
which is a $U$-statistic of degree $2$ with kernel $f(x_i,x_j)=|x_i - x_j|$.
Gini coefficient, the most commonly used measure of inequality,
 is obtained by multiplying $G$ by $(n-1)/2\sum_{i=1}^nx_i$.

\begin{remark}
The variance of a sample, obtained by replacing the absolute difference by the
squared difference in \eqref{eq:gini}, is also a $U$-statistic. However we
note that computing the variance can be achieved by computing two sums of
locally computable variables ($x_i$ and $x_i^2$), which can be done with
existing LDP protocols.
\end{remark}



\textbf{Rényi-2 entropy.} Also known as collision entropy, this provides a
measure of entropy between Shannon's entropy and min entropy which is used
in many applications involving discrete distributions \citep[see][and
references therein]{Acharya2015}.
It is given by
\begin{equation}
  H_2=-\ln\Big(\textstyle\frac{2}{n(n-1)}\sum_{i<j}\mathbb{I}[x_i=x_j] \Big).
\end{equation}
The expression inside the log is a $U$-statistic of degree $2$ with kernel $f
(x_i,x_j)=\mathbb{I}[x_i=x_j]$.

\textbf{Kendall's tau coefficient.} This statistic measures the
ordinal association between two variables and is often used as a test
statistic to answer questions such as ``does a higher salary make one
happier?''. In learning to rank applications, it is used to
evaluate the extent to which a predicted ranking correlates with the
(human-generated) gold standard \citep[see e.g.,][]
{svm_rank,kendall_nlp}.
Formally, assuming continuous variables for simplicity,
let $
\mathcal{X}\subset\mathbb{R}^2$ and $\dataset=\{x_i=(y_i,z_i)\}_{i=1}^n$.
For
any $i<j$,
the pairs $x_i=(y_i,z_i)$ and $x_j=(y_j,z_j)$ are
said be \textit{concordant} if $(y_i > y_j) \wedge (z_i > z_j)$ or $(y_i <
y_j) \wedge (z_i < z_j)$, and \textit{discordant} otherwise.
Let $C$ and $D$ be the number of concordant and discordant pairs in
$\dataset$.
Kendall rank correlation coefficient is defined as:
\begin{align}
\tau = \frac{C-D}{C+D}=\frac{1}{{n \choose 2}}\sum_{i <j} \sgn(y_i-y_j)\sgn(z_i-z_j),
\label{eq:kendall}
\end{align} 
which is a $U$-statistic of degree $2$ with kernel $f(x_i,x_j)=\sgn
(y_i-y_j)\sgn(z_i-z_j)$.\footnote{One can easily modify the kernel to
account for ties.}

\textbf{Area under the ROC curve (AUC).} In binary classification with
class imbalance, the Receiver Operating Characteristic (ROC) gives the true
positive rate with respect to the false positive rate of a predictor
at each possible decision threshold. The AUC is a popular summary of the ROC
curve which gives a single, threshold-independent measure of the classifier
goodness: it corresponds to the probability that the predictor assigns a
higher score to a randomly chosen positive point than to a randomly chosen
negative one. AUCs are widely used as performance metrics in machine learning
\citep{auc,auc2}, and have also been recently studied as fairness measures
\citep{NIPS2019_8604,fairness_scoring}.
Formally, let $\mathcal{X}\subset\mathbb{R}\times\{-1,1\}$ and $\dataset=
\{x_i=(s_i,y_i)\}_{i=1}^n$ where for each data point $i$, $s_i\in
\mathbb{R}$ is the score assigned to point $i$ and $y_i\in\{-1,1\}$ is its
label. For convenience, let
$\dataset^{+}=\{s_i: y_i=1 \}$ and $\dataset^{-}=\{s_i: y_i=-1 \}$ and let $n^+=|\dataset^+|$ and $n^-=|\dataset^-|$.
The AUC is given by
\begin{align}
\label{eq:auc_normal}
AUC &= \textstyle\frac{1}{n^+n^-}\sum_{s_i\in\dataset^+}\sum_
{s_j\in\dataset^-} 
\mathbb{I}[s_i > s_j],
\end{align}
where $\mathbb{I}[\sigma]$ is an indicator variable outputting $1$ if the
predicate $\sigma$ is true and $0$ otherwise. Up to a ${n
\choose 2}/n^{+} n^{-}$ factor, it is easy to see that $AUC$ is a
$U$-statistic of degree $2$ with kernel $f(x_i,x_j) = \mathbb{I}[s_i > s_j
\wedge y_i>y_j]+\mathbb{I}[s_i < s_j \wedge y_i<y_j]$.

\textbf{Machine learning with pairwise losses.} Many machine learning problems
involve loss functions that operate on pairs of points 
\citep{pairwise-losses,CBC2016}. This is the case for
instance in metric learning \citep{Bellet2015c}, bipartite ranking \citep{CLV08}
and clustering \citep{CLEM14}. Empirical risk minimization problems have therefore
the following generic form:
\begin{align}
\label{eq:erm}
\min_{\theta\in\Theta} \textstyle\frac{2}{n(n-1)} \sum_{i<j}\ell_\theta
(x_i,x_j),
\end{align}
where $\theta\in\Theta$ are model parameters. The objective function in 
\eqref{eq:erm}, as well as its gradient, are $U$-statistics of degree $2$ with
kernels $\ell_\theta$ and $\nabla_\theta\ell_\theta$ respectively.

\subsection{Local Differential Privacy}

The classic \emph{centralized} model of differential privacy assumed the
presence of a trusted aggregator which processes the private information of
individuals and releases a perturbed version of the result.
The \emph{local} model instead captures the setting where individuals do
not trust the aggregator and randomize their input locally before sharing it.
This model has received wide industrial adoption
\citep{rappor:15,rappor2:16,applewhitepaper:17,telemetry:17}.

\begin{definition}[\citealp{d:13}]
A local randomizer $\lr$ is $(\epsilon,\delta)$-locally
differentially private (LDP) if for all $x, x'\in \mathcal{X}$ and all
possible output $O$ in
the range of $\lr$:
\begin{align*}
Pr[\lr(x)=O] \leq e^{\epsilon} Pr[\lr(x')=O] + \delta.
\end{align*} 
The special case $\delta=0$ is called pure $\epsilon$-LDP.
\end{definition}

Most work in LDP aims to compute quantities that are separable across
individual inputs, such as sums and histograms \citep[see][and
references therein]
{sb:15,Wang:Blocki:Li:Jha:17,KCS:2019,CKS:18,HeavyHitters:17}. In contrast,
our goal is to design LDP protocols for computing $U$-statistics, where each
term involves a pair of inputs.

%% file: subfiles/generic_ldp.tex

\if\arxiv1
\section{Generic LDP Protocol from Quantization}
\else
\section{GENERIC LDP PROTOCOL FROM QUANTIZATION}
\fi
\label{sec:generic_ldp}

\begin{algorithm2e}[t]
  \DontPrintSemicolon
  \LinesNumbered
  \SetKwComment{Comment}{{\scriptsize$\triangleright$\ }}{}
  \caption{LDP algorithm based on quantization and private histograms}
  \label{algo:DP-hist}
  {\bf Public Parameters:}~ Privacy budget $\epsilon$,
  quantization scheme $\pi$, number of bins $k$.\\
  \KwIn{$(x_i \in \mathcal{X})_{i\in[n]}$}
  \KwOut{Estimate $\widehat{U}_{f,n}$ of $U_f$}
  \BlankLine

      \For{each user $i\in[n]$}{

      Form quantized input $\pi(x_i)\in[k]$

      For $\beta=k/(k+e^\epsilon-1)$, generate $\tilde{x}_i\in[k]$ s.t.
      \begin{equation}
      \label{eq:randmized-response}
      P(\tilde{x}_i=i) = \left\{\begin{array}
      {ll}1-\beta &
      \text{for }i = \pi
      (x_i),\\ \beta/k & \text{for }i \neq \pi
      (x_i),\end{array}\right.
      \end{equation}

      Send $\tilde{x}_i$ to the aggregator
      }

      Return $\widehat{U}_{f,n}$ computed from $\tilde{x}_1,\dots,\tilde{x}_n$
      and $\beta$
\end{algorithm2e}

\textbf{Discrete inputs.} We first consider the case of discrete inputs taking
one of $k$ values. The possible values of the kernel function can be
written as a matrix $A\in\mathbb{R}^{k\times k}$ where $A_{ij}=f(i,j)$. In
this case, we can set the local randomizer $\RRR$ to be $k$-ary randomized
response to generate a perturbed version $\RRR(x_i)$ of each input $x_i$. Let $e_i$
denote the vector of length $k$ with a one in the $i$-th position and $0$
elsewhere. For each perturbed input in one-hot encoding form $e_{\RRR(x_i)}$ we
can deduce an unbiased estimate of $e_{x_i}$.
As the discrete $U$-statistic is a linear function of each of these vectors,
computing it on these unbiased estimates gives an unbiased estimate $
\widehat{U}_{f,n}$ which can be written as:
$$\widehat{U}_{f,n}=\frac{1}{\binom{n}{2}}\sum_{1\leq i<j\leq n}\widehat{f}_A(\RRR
  (x_i),\RRR(x_j)),$$
and is itself a $U$-statistic with kernel $\widehat{f}_A$ given by
\begin{equation*}
  \widehat{f}_A(\RRR(x_1),\RRR(x_2))=(1-\beta)^{-2}(e_{\RRR(x_1)}-b)^TA(e_
  {\RRR
  (x_2)}-b),
\end{equation*}
where $1-\beta$ is the probability of returning the true input in $k$-ary
randomized response (see Eq. \ref{eq:randmized-response}) and $b$ is the
vector of length $k$ with every entry $\beta/k$.
Details and analysis of this process, leveraging Hoeffding's decomposition of
$U$-statistics \citep{Hoeffding48,Lee90}, can be found in Section~\ref{sec:rr-discrete} of
the supplementary material. The resulting bounds on the variance of $
\widehat{U}_{f,n}$ are summarized in the following theorem.

\begin{theorem}
\label{thm:discrete-case}
  If $f(x,x')\in [0,1]$ for all $x,x'$, then
  \begin{equation*}
    \var(\widehat{U}_{f,n})\leq \frac{1}{n(1-\beta)^2}+\frac{(1+\beta)^2}{2n(n-1)(1-\beta)^4}.
  \end{equation*}
  In order to achieve $\epsilon$-LDP with a fixed $k$ this becomes,
  \begin{equation*}
    \var(\widehat{U}_{f,n}) \approx \frac{(1+k/\epsilon)^2}{n}+\frac{(1+k/\epsilon)^4}
    {2n^2} \approx \frac{k^2}{n\epsilon^2},
  \end{equation*}
\end{theorem}

\textbf{Continuous inputs.} For $U$-statistics on discrete domains, e.g.
Renyi-2 entropy, the above strategy can be applied directly.
Possibly more importantly however, this then leads to a natural protocol for the continuous case.
In this protocol (see Algorithm~\ref{algo:DP-hist}),
the local randomizer proceeds by quantizing the input into $k$ bins (for
instance using simple or randomized rounding) before applying the previous
procedure.

There are two sources of error in this protocol. The first one is due to the
randomization needed to satisfy LDP in the quantized domain as bounded in
Theorem~\ref{thm:discrete-case}.
The second
source of error is due to quantization.
In order to control this error in a nontrivial way, we rely on an assumption on the kernel function (namely, that it is Lipschitz) or the data
distribution (namely, that it has Lipschitz density). Under these assumptions and using an
appropriate variant of the kernel function on the quantized domain, we show
that we can bound the error with respect to the original domain by a term in
$O(1/k^2)$ (see Section~\ref{sec:discretization} of the supplementary
material). This leads to the
following result.

\begin{theorem}
\label{thm:dphist-main}
For simplicity, assume bounded domain $\mathcal{X}=[0,1]$ and kernel values $f
(x,y)\in[0,1]$ for all $x,y\in\mathcal{X}$. Let $\pi$ correspond to simple
rounding,
$\epsilon>0$, $k\geq 1$ and $\beta=k/(k+e^\epsilon-1)$. Then Algorithm~\ref{algo:DP-hist} satisfies $\epsilon$-LDP. Furthermore:
\begin{itemize}
\item If $f$ is $L_f$-Lipschitz in each of its arguments, then $\mse(\widehat{U}_{f,n})$ is less than or equal to
\begin{equation*}
 \frac{1}{n(1-\beta)^2}+\frac{(1+\beta)^2}{2n(n-1)
(1-\beta)^4} + \frac{L_f^2}{2k^2}.
\end{equation*}
\item If $d\mu/d\lambda$ is $L_\mu$-Lipschitz w.r.t. some measure $\lambda$,
then $\mse(\widehat{U}_{f,n})$ is less than or equal to
\begin{equation*}
\frac{1}{n(1-\beta)^2}+\frac{(1+\beta)^2}{2n(n-1)
(1-\beta)^4} + \frac{4L_\mu^2}{k^2} + \frac{4L_\mu^4}{k^4}.
\end{equation*}
\end{itemize}
\end{theorem}

\begin{remark}
The use of simple rounding is not optimal in many situations. In the case of
sum, and possibly of the Gini coefficient, it would be more accurate if randomized rounding were used instead of simple rounding. We leave this investigation for later work.
\end{remark}

Setting $k$ so as to balance the quantization and estimation errors leads to
the following corollary.

\begin{corollary}
\label{cor:dphist-main}
Under the conditions of Theorem~\ref{thm:dphist-main}, for $\epsilon\leq 1$
and large enough $n$, taking $k=n^{1/4}\sqrt{L\epsilon}$ leads to
$\mse(\widehat{U}_{f,n}) = O(L/\sqrt{n}\epsilon)$, where $L$ corresponds to $L_f$ or
$L_\mu$ depending on the assumption.
\end{corollary}

This result gives concrete error bounds for $U$-statistics whose
kernel is Lipschitz, for arbitrary data distributions. One important example
is the Gini mean difference, whose corresponding kernel
$f(x_i,x_j)=|x_i -x_j|$ is $1$-Lipschitz. On the other hand, for
$U$-statistics with non-Lipschitz kernels, the data distribution must be
sufficiently smooth (if not, it is easy to construct cases that make the
algorithm fail).






%% file: subfiles/auc.tex

\if\arxiv1
\section{Locally Private AUC}
\else
\section{LOCALLY PRIVATE AUC}
\fi
\label{sec:auc}


In this section, we describe an algorithm for computing AUC 
\eqref{eq:auc_normal}, whose
kernel is discontinuous and therefore non-Lipschitz. We assume $\mathcal{X}$
to be an ordered domain 
of size $d$, that is with each datum in $[0..d-1]$. Note that all data is in
practice discrete when represented in finite precision, so this is general.
For simplicity of presentation we will assume that (i) $d=2^{\alpha}$ for some integer $\alpha$, and (ii) that the classes of the data, the $y_i$, are public.

Our solution for computing AUC in the local model relies on a hierarchical histogram construction that has been considered in previous works for private collection of high-dimensional data~\citep{DBLP:conf/fc/ChanSS12}, heavy hitters~\citep{HeavyHitters:17}, and range queries~\citep{KCS:2019}.
A hierarchical histogram is essentially a tree data structure on top of a histogram where each internal node is labelled with the sum of the values in the interval covered by it (see Figure~\ref{fig:hh-example}). That allows to answer any range query about $u$ by checking the value associated with $O(\log |u|)$ nodes in the tree. We first define an exact version of such hierarchical histograms and explain how to compute AUC from one.

\paragraph{Notation on trees.}
We represent a binary tree $h$ of depth $\alpha$ with integer node labels as a total mapping from a prefix-closed set of binary strings of length at most $\alpha$ to the integers. We refer to the $i$-th node in level $l$ of the tree by the binary representation of $i$ padded to length $l$ from the left with zeros. With this notation, $h_\lambda$ is the label of the root node, as we use $\lambda$ to denote the empty string, $h_0$ (resp. $h_1$) is the integer label of the left (resp. right) child of the root of $h$, and in general $h_p$ is the label of the node at path $p$ from the root, i.e. the label of the node reached by following left or right children from the root according to the value of $p$ ($0$ indicates left and $1$ indicates right). Let $b_i$ be the $i$-th node in the bottom level. For two binary strings $p, p'\in \{0, 1\}^*$ we denote the prefix relation by $p' \preceq p$, and their concatenation as $p\cdot p'$.

\begin{definition}
  Let $S = \{s_1, \ldots, s_n\}$ be a multiset, with $s_i\in [0..d-1]$.
  A hierarchical histogram of $S$ is a total mapping  $h:\{0,1\}^{\leq \log(d)}\to \mathbb{Z}$ defined as $h(b) = |\{s\in S~|~ \exists b'\in\{0,1\}^*: b\cdot b' = b_s\}|$. For simplicity, we denote $h(b)$ by $h_b$.
  \label{def:hierhist}
\end{definition}

\if\arxiv1
  \def\figsize{.5}
\else
  \def\figsize{.85}
\fi

\begin{figure}
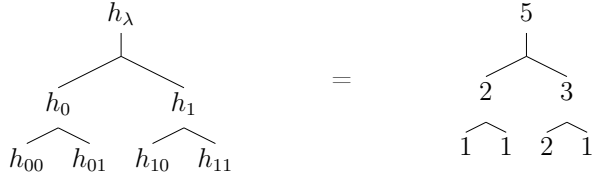

    \centering
    \resizebox{\figsize\columnwidth}{!}{%
      \begin{tabular}{ccc}
        \Tree [.$h_\lambda$ [[.$h_0$ $h_{00}$ $h_{01}$ ] [.$h_1$ $h_{10}$ $h_{11}$ ] ] ] &
        ~~~~~~~~\multirow{5}{*} {=}~~~~~~~~~~&
        \Tree [.$5$ [[.$2$ $1$ $1$ ] [.$3$ $2$ $1$ ] ] ]
      \end{tabular}
    }
    \caption{Hierarchical histogram $h$ for multiset $\{0, 1, 2, 2, 3\}$ over the domain $\{0, 1, 2, 3\}$.}
\label{fig:hh-example}
\end{figure}
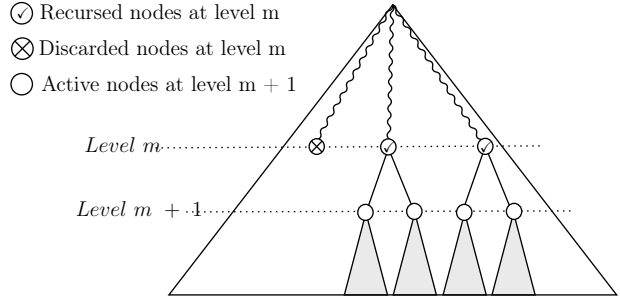
\begin{figure}
\centering
\resizebox{\figsize\columnwidth}{!}{%

\tikzset{every picture/.style={line width=0.75pt}} 

\begin{tikzpicture}[x=0.75pt,y=0.75pt,yscale=-1,xscale=1]

\draw   (300.11,22.2) -- (467.3,239) -- (132.92,239) -- cycle ;
\draw  [dash pattern={on 0.84pt off 2.51pt}]  (145.81,176.95) -- (435.8,176.2) ;

\draw  [dash pattern={on 0.84pt off 2.51pt}]  (122.89,129.1) -- (410.73,127.61) ;

\draw    (299.75,22.2) .. controls (300.4,24.47) and (299.59,25.92) .. (297.32,26.57) .. controls (295.05,27.22) and (294.24,28.67) .. (294.89,30.94) .. controls (295.54,33.21) and (294.73,34.66) .. (292.46,35.31) .. controls (290.19,35.96) and (289.38,37.41) .. (290.03,39.68) .. controls (290.68,41.95) and (289.87,43.4) .. (287.6,44.05) .. controls (285.33,44.7) and (284.52,46.15) .. (285.17,48.42) .. controls (285.82,50.69) and (285.01,52.14) .. (282.74,52.79) .. controls (280.47,53.44) and (279.67,54.89) .. (280.32,57.16) .. controls (280.97,59.43) and (280.16,60.88) .. (277.89,61.53) .. controls (275.62,62.18) and (274.81,63.63) .. (275.46,65.9) .. controls (276.11,68.17) and (275.3,69.62) .. (273.03,70.27) .. controls (270.76,70.92) and (269.95,72.37) .. (270.6,74.64) .. controls (271.25,76.91) and (270.44,78.36) .. (268.17,79.01) .. controls (265.9,79.66) and (265.09,81.11) .. (265.74,83.38) .. controls (266.39,85.65) and (265.58,87.1) .. (263.31,87.75) .. controls (261.04,88.4) and (260.23,89.85) .. (260.88,92.12) .. controls (261.53,94.39) and (260.72,95.84) .. (258.45,96.49) .. controls (256.18,97.14) and (255.37,98.59) .. (256.02,100.86) .. controls (256.67,103.13) and (255.86,104.58) .. (253.59,105.23) .. controls (251.32,105.88) and (250.51,107.33) .. (251.16,109.6) .. controls (251.81,111.87) and (251,113.32) .. (248.73,113.97) .. controls (246.46,114.62) and (245.65,116.07) .. (246.3,118.34) .. controls (246.95,120.61) and (246.14,122.06) .. (243.87,122.71) .. controls (241.6,123.36) and (240.79,124.81) .. (241.44,127.08) -- (240.32,129.1) -- (240.32,129.1) ;

\draw    (299.75,22.2) .. controls (301.37,23.91) and (301.33,25.58) .. (299.62,27.2) .. controls (297.91,28.82) and (297.86,30.49) .. (299.48,32.2) .. controls (301.1,33.91) and (301.06,35.57) .. (299.35,37.19) .. controls (297.64,38.82) and (297.6,40.48) .. (299.22,42.19) .. controls (300.84,43.9) and (300.79,45.57) .. (299.08,47.19) .. controls (297.37,48.82) and (297.33,50.48) .. (298.95,52.19) .. controls (300.57,53.9) and (300.53,55.56) .. (298.82,57.19) .. controls (297.11,58.81) and (297.06,60.48) .. (298.68,62.19) .. controls (300.3,63.9) and (300.26,65.56) .. (298.55,67.18) .. controls (296.84,68.81) and (296.8,70.47) .. (298.42,72.18) .. controls (300.04,73.89) and (299.99,75.56) .. (298.28,77.18) .. controls (296.57,78.81) and (296.53,80.47) .. (298.15,82.18) .. controls (299.77,83.89) and (299.73,85.55) .. (298.02,87.18) .. controls (296.31,88.8) and (296.26,90.47) .. (297.88,92.18) .. controls (299.5,93.89) and (299.46,95.55) .. (297.75,97.17) .. controls (296.04,98.8) and (296,100.46) .. (297.62,102.17) .. controls (299.24,103.88) and (299.19,105.55) .. (297.48,107.17) .. controls (295.77,108.8) and (295.73,110.46) .. (297.35,112.17) .. controls (298.97,113.88) and (298.93,115.54) .. (297.22,117.17) .. controls (295.51,118.78) and (295.46,120.45) .. (297.08,122.16) .. controls (298.7,123.87) and (298.66,125.53) .. (296.95,127.16) -- (296.89,129.48) -- (296.89,129.48) ;

\draw    (299.75,22.2) .. controls (302.06,22.68) and (302.97,24.07) .. (302.49,26.38) .. controls (302.01,28.69) and (302.92,30.08) .. (305.23,30.57) .. controls (307.54,31.05) and (308.45,32.44) .. (307.96,34.75) .. controls (307.48,37.06) and (308.39,38.45) .. (310.7,38.94) .. controls (313.01,39.42) and (313.92,40.81) .. (313.44,43.12) .. controls (312.96,45.43) and (313.87,46.82) .. (316.18,47.3) .. controls (318.49,47.79) and (319.4,49.18) .. (318.91,51.49) .. controls (318.43,53.8) and (319.34,55.19) .. (321.65,55.67) .. controls (323.96,56.16) and (324.87,57.55) .. (324.39,59.86) .. controls (323.91,62.17) and (324.82,63.56) .. (327.13,64.04) .. controls (329.44,64.52) and (330.35,65.91) .. (329.86,68.22) .. controls (329.38,70.53) and (330.29,71.92) .. (332.6,72.41) .. controls (334.91,72.89) and (335.82,74.28) .. (335.34,76.59) .. controls (334.86,78.9) and (335.77,80.29) .. (338.08,80.78) .. controls (340.39,81.26) and (341.3,82.65) .. (340.81,84.96) .. controls (340.33,87.27) and (341.24,88.66) .. (343.55,89.14) .. controls (345.86,89.63) and (346.77,91.02) .. (346.29,93.33) .. controls (345.8,95.64) and (346.71,97.03) .. (349.02,97.51) .. controls (351.33,98) and (352.24,99.39) .. (351.76,101.7) .. controls (351.28,104.01) and (352.19,105.4) .. (354.5,105.88) .. controls (356.81,106.37) and (357.72,107.76) .. (357.24,110.07) .. controls (356.75,112.38) and (357.66,113.77) .. (359.97,114.25) .. controls (362.28,114.73) and (363.19,116.12) .. (362.71,118.43) .. controls (362.23,120.74) and (363.14,122.13) .. (365.45,122.62) .. controls (367.76,123.1) and (368.67,124.49) .. (368.19,126.8) -- (369.2,128.36) -- (369.2,128.36) ;

\draw    (296.89,129.48) -- (316.22,176.95) ;

\draw    (296.89,129.48) -- (280.06,177.7) ;

\draw  [fill={rgb, 255:red, 155; green, 155; blue, 155 }  ,fill opacity=0.21 ] (280.06,177.7) -- (296.17,239) -- (263.95,239) -- cycle ;
\draw  [fill={rgb, 255:red, 155; green, 155; blue, 155 }  ,fill opacity=0.21 ] (316.58,177.7) -- (332.69,239) -- (300.47,239) -- cycle ;
\draw  [color={rgb, 255:red, 0; green, 0; blue, 0 }  ,draw opacity=1 ][fill={rgb, 255:red, 255; green, 255; blue, 255 }  ,fill opacity=1 ] (310.49,176.58) .. controls (310.49,173.27) and (313.06,170.6) .. (316.22,170.6) .. controls (319.38,170.6) and (321.95,173.27) .. (321.95,176.58) .. controls (321.95,179.88) and (319.38,182.56) .. (316.22,182.56) .. controls (313.06,182.56) and (310.49,179.88) .. (310.49,176.58) -- cycle ;
\draw  [color={rgb, 255:red, 0; green, 0; blue, 0 }  ,draw opacity=1 ][fill={rgb, 255:red, 255; green, 255; blue, 255 }  ,fill opacity=1 ] (273.97,177.32) .. controls (273.97,174.02) and (276.54,171.34) .. (279.7,171.34) .. controls (282.87,171.34) and (285.43,174.02) .. (285.43,177.32) .. controls (285.43,180.63) and (282.87,183.3) .. (279.7,183.3) .. controls (276.54,183.3) and (273.97,180.63) .. (273.97,177.32) -- cycle ;
\draw    (369.2,127.98) -- (388.54,175.46) ;

\draw    (369.2,127.98) -- (352.38,176.2) ;

\draw  [fill={rgb, 255:red, 155; green, 155; blue, 155 }  ,fill opacity=0.21 ] (353.81,177.7) -- (369.92,239) -- (337.7,239) -- cycle ;
\draw  [fill={rgb, 255:red, 155; green, 155; blue, 155 }  ,fill opacity=0.21 ] (390.33,177.7) -- (406.44,239) -- (374.22,239) -- cycle ;
\draw  [color={rgb, 255:red, 0; green, 0; blue, 0 }  ,draw opacity=1 ][fill={rgb, 255:red, 255; green, 255; blue, 255 }  ,fill opacity=1 ] (384.96,176.58) .. controls (384.96,173.27) and (387.52,170.6) .. (390.69,170.6) .. controls (393.85,170.6) and (396.41,173.27) .. (396.41,176.58) .. controls (396.41,179.88) and (393.85,182.56) .. (390.69,182.56) .. controls (387.52,182.56) and (384.96,179.88) .. (384.96,176.58) -- cycle ;
\draw  [color={rgb, 255:red, 0; green, 0; blue, 0 }  ,draw opacity=1 ][fill={rgb, 255:red, 255; green, 255; blue, 255 }  ,fill opacity=1 ] (347.72,177.32) .. controls (347.72,174.02) and (350.29,171.34) .. (353.45,171.34) .. controls (356.62,171.34) and (359.18,174.02) .. (359.18,177.32) .. controls (359.18,180.63) and (356.62,183.3) .. (353.45,183.3) .. controls (350.29,183.3) and (347.72,180.63) .. (347.72,177.32) -- cycle ;
\draw  [color={rgb, 255:red, 0; green, 0; blue, 0 }  ,draw opacity=1 ][fill={rgb, 255:red, 255; green, 255; blue, 255 }  ,fill opacity=1 ] (15,81.5) .. controls (15,77.08) and (18.58,73.5) .. (23,73.5) .. controls (27.42,73.5) and (31,77.08) .. (31,81.5) .. controls (31,85.92) and (27.42,89.5) .. (23,89.5) .. controls (18.58,89.5) and (15,85.92) .. (15,81.5) -- cycle ;
\draw  [color={rgb, 255:red, 0; green, 0; blue, 0 }  ,draw opacity=1 ][fill={rgb, 255:red, 255; green, 255; blue, 255 }  ,fill opacity=1 ] (15,28.5) .. controls (15,24.08) and (18.58,20.5) .. (23,20.5) .. controls (27.42,20.5) and (31,24.08) .. (31,28.5) .. controls (31,32.92) and (27.42,36.5) .. (23,36.5) .. controls (18.58,36.5) and (15,32.92) .. (15,28.5) -- cycle ;

\draw  [color={rgb, 255:red, 0; green, 0; blue, 0 }  ,draw opacity=1 ][fill={rgb, 255:red, 255; green, 255; blue, 255 }  ,fill opacity=1 ] (291.07,128.38) .. controls (291.07,125.08) and (293.64,122.4) .. (296.8,122.4) .. controls (299.97,122.4) and (302.53,125.08) .. (302.53,128.38) .. controls (302.53,131.69) and (299.97,134.36) .. (296.8,134.36) .. controls (293.64,134.36) and (291.07,131.69) .. (291.07,128.38) -- cycle ;

\draw  [color={rgb, 255:red, 0; green, 0; blue, 0 }  ,draw opacity=1 ][fill={rgb, 255:red, 255; green, 255; blue, 255 }  ,fill opacity=1 ] (363.48,127.98) .. controls (363.48,124.68) and (366.04,122) .. (369.2,122) .. controls (372.37,122) and (374.93,124.68) .. (374.93,127.98) .. controls (374.93,131.29) and (372.37,133.96) .. (369.2,133.96) .. controls (366.04,133.96) and (363.48,131.29) .. (363.48,127.98) -- cycle ;

\draw  [fill={rgb, 255:red, 255; green, 255; blue, 255 }  ,fill opacity=1 ] (237.65,128.1) .. controls (237.65,124.88) and (240.19,122.27) .. (243.32,122.27) .. controls (246.45,122.27) and (248.99,124.88) .. (248.99,128.1) .. controls (248.99,131.33) and (246.45,133.94) .. (243.32,133.94) .. controls (240.19,133.94) and (237.65,131.33) .. (237.65,128.1) -- cycle ; \draw   (239.31,123.98) -- (247.33,132.23) ; \draw   (247.33,123.98) -- (239.31,132.23) ;
\draw  [fill={rgb, 255:red, 255; green, 255; blue, 255 }  ,fill opacity=1 ] (15.52,54.93) .. controls (15.52,50.51) and (19.1,46.93) .. (23.52,46.93) .. controls (27.94,46.93) and (31.52,50.51) .. (31.52,54.93) .. controls (31.52,59.35) and (27.94,62.93) .. (23.52,62.93) .. controls (19.1,62.93) and (15.52,59.35) .. (15.52,54.93) -- cycle ; \draw   (17.86,49.27) -- (29.18,60.59) ; \draw   (29.18,49.27) -- (17.86,60.59) ;

\draw (98.55,126.86) node   {${\displaystyle Level\ m}$};
\draw (110.72,174.71) node   {${\displaystyle Level\ m\ +\ 1}$};
\draw (425.77,46.87) node [scale=0.9] [align=left] {};
\draw (24,28.5) node [scale=0.7,xslant=0.05]  {$\checkmark $};
\draw (297.52,128.38) node [scale=0.7,xslant=0.05]  {$\checkmark $};
\draw (369.92,127.98) node [scale=0.7,xslant=0.05]  {$\checkmark $};
\draw (121,27) node [scale=1] [align=left] {~~Recursed nodes at level m$ $};
\draw (128,82) node [scale=1] [align=left] {~~Active nodes at level m + 1$ $$
$};
\draw (123,54) node [scale=1] [align=left] {~~Discarded nodes at level m$ $};

\end{tikzpicture}
  }
\caption{Our algorithm can be seen as a
breath-first traversal of a tree, where at each level some nodes are selected
for their subtrees to be explored further.}
\label{fig:algo}
\end{figure}


\textbf{Algorithm.} We use hierarchical histograms to compute AUC as follows.
Let $S^+$ and $S^-$ be the samples of the positive and negative classes from which we want to estimate AUC. Let $h^+$ and $h^-$ be hierarchical histograms for $S^+$ and $S^-$. Note that $h^+_\lambda = n^+$ and $h^-_\lambda = n^-$. We can now define the unnormalized AUC, denoted UAUC,  over hierarchical histograms recursively by letting $\uauc(h^+, h^-,p)$ be $0$, if $p$ is a leaf, and otherwise setting:
\begin{equation*}
  \uauc(h^+, h^-,p) = h^+_{p\cdot 1} h^-_{p\cdot 0} + \sum\limits_{i\in\{0,1\}} \uauc(h^+, h^-,p\cdot i) \enspace.
  \label{eq:uauc-noprivacy}
\end{equation*}
Thus we have $\auc(S^+, S^-) = \auc(h^+, h^-, \lambda) = \frac{1}{n^+n^-} \uauc(h^+, h^-, \lambda)$.

The above definition naturally leads to an algorithm that proceeds by
traversing the trees $h^+, h^-$ top-down from the root $\lambda$, accumulating
the products of counts from $h^+, h^-$ at nodes that correspond to entries in $h^+$ that are bigger than entries in $h^-$. We now define a differentially private analogue. Later we will describe an efficient frequency oracle which can be used to compute an LDP estimate $\hat{h}$ of a hierarchical histogram $h$ of $n$ values in a domain of size $2^\alpha$. This will provide the following necessary properties (i) $\hat{h}$ is {\em unbiased}, (ii) $\var(\hat{h})\leq v$, with $v$ defined as $Cn\alpha$ for some small constant $C$ (iii) the $\hat{h}_p$ are pairwise independent and (iv) Each level of $\hat{h}$ is independent of the other levels. Our private algorithm for computing an estimate of UAUC is then defined in terms of parameters $n^+$ and $n^-$, $v^+$ and $v^-$ (bounding the variance of $\hat{h}^+$ and $\hat{h}^-$ respectively), and $a > 1$ is a small number depending on $n^+$, $n^-$, $\alpha$ and $C$.

For a symbol $\aleph$ we write $\aleph^\pm$ to simultaneously refer to $\aleph^+$ and $\aleph^-$.
Let $\tilde{h}^\pm_p=\max(\hat{h}^\pm_p,\sqrt{av^{\pm}}/2)$, i.e.
$\tilde{h}^+_p=\max(\hat{h}^+_p,\sqrt{av_{+}}/2)$ and $\tilde{h}^-_p=\max(
\hat{h}^-_p,\sqrt{av_{-}}/2)$, and let $\tau=a\sqrt{v^-v^+}$. Our private
estimate is defined as follows. If $p$ is a leaf then $\uaucest(\hat{h}^+, \hat{h}^-, p)$ is $0$, else if $\tilde{h}^+_{p} \tilde{h}^-_{p} < \tau$ then it is given by
\begin{equation*}
   \textstyle\frac{1}{2} \sum_{i\in\{0,1\}} \hat{h}^+_{p\cdot i}
   \sum_{i\in\{0,1\}}\hat{h}^-_{p\cdot i} \enspace.
\end{equation*}
Otherwise, it is given recursively by
\begin{equation}
  \hat{h}^+_{p\cdot 1} \hat{h}^-_{p\cdot 0} + \textstyle\sum_{i\in\{0,1\}} \uaucest(\hat{h}^+, \hat{h}^-, p\cdot i) \enspace.
  \label{eq:uauc-privacy}
\end{equation}

As before, this definition leads to an algorithm. Note that the only difference with its non-private analogue is that this procedure does not recurse into subtrees whose contribution to the UAUC is upper bounded sufficiently tightly. More concretely, the server starts by querying $\hat{h}^+, \hat{h}^-$ at the root, namely with $p = \lambda$. If $p$ is a leaf then we return $0$ as the AUC. Otherwise, the algorithm checks whether $\tilde{h}^+_{p} \tilde{h}^-_{p}<\tau$. If so, then the algorithm concludes that there is not much to gain in exploring the subtrees rooted at $p\cdot 0$ and $p\cdot 1$, and returns $\frac{1}{2} \sum_{i\in{0,1}} \hat{h}^+_{p\cdot i} \sum_{i\in{0,1}}\hat{h}^-_{p\cdot i}$ as an estimate of $\frac{1}{2} h^+_p h^-_p$. This estimate might seem equivalent to $\frac{1}{2} \hat{h}^+_{p}\hat{h}^-_{p}$, but takes the previous form for a technical reason that is made clear in the proof. In this case we call $p$ a {\em discarded} node. On the other hand, if $\tilde{h}^+_{p} \tilde{h}^-_{p}\geq\tau$, the algorithm proceeds as its non-private analogue, accumulating the contribution to the $\uauc$ from the direct subtrees of $p$ and recursing into nodes $p\cdot 0$ and $p\cdot 1$. In this case we refer to $p$ as a {\em recursed} node. Thus every node $p\in\{0,1\}^{\leq \alpha}$ will be either recursed, a leaf or there will be a discarded node $p'$ such that $p'\preceq p$. This is depicted in Figure~\ref{fig:algo}.


\textbf{Analysis.} Note that our algorithm has two sources of error: (i) the
one incurred by discarding nodes and (ii) the error in estimating the contribution to the UAUC of the recursed nodes.
The threshold $\tau$ is carefully chosen to balance these two errors.

Let $\recursed^m$ be the set of nodes recursed on at level $m$.
Our accuracy proof starts by bounding the expected value of $|\recursed^m|$ 
(see Lemma~\ref{lem:recursed} in Section~\ref{sec:auc-proof} of the
supplementary) by a quantity
$B$ that is independent of $m$. We now describe a central argument to our
accuracy proof, stated in the next theorem. Let $E^{\recursed}_m$ be the
contribution to the error by nodes in $\recursed^m$. Then, the total
contribution to the error by recursed nodes is $E^\recursed = \sum_{m\in 
[\alpha]}E^{\recursed}_m$. A useful identity is $\EE({E^\recursed}^2)=\sum_
{m\in [\alpha]}\EE({E^\recursed_m}^2)$, as we can bound $\EE(
{E^\recursed_m}^2)$, for any $m$, in terms of $B$ (see detailed proof in the
supplementary). Note that this identity follows from $\EE
(E^\recursed_mE^\recursed_{m'}) = 0$, with $m'> m$. The latter would hold if errors $E^\recursed_m$ and $E^\recursed_{m'}$ were independent, since our frequency oracle is unbiased. However, errors at a given level are not independent of previous levels. However $\EE(E^\recursed_mE^\recursed_{m'}) = 0$ because the conditional expectation of $E^\recursed_{m'}$ with respect to the answers of the frequency oracles up to level $m'$ is $0$ i.e. $E^\recursed_1, \ldots, E^\recursed_{m'}$ is a martingale difference sequence. The idea of conditioning on previous levels is used several times in our proofs, also to bound the error due to discarded nodes.

Next, we state our accuracy result, which is proven in detail in Section~\ref{sec:auc-proof} of the supplementary. Our proof tracks constants: this is
important for
practical purposes, and we show empirically in Section~\ref{sec:synthetic-experiments} that our bound is in fact
quite tight.

\begin{theorem}
  \label{thm:aucacc}
  If $\alpha\leq \sqrt{n}$ and the following holds:
  \begin{enumerate}
  \item $\EE(\hat{h}^{\pm}_{p}-h^{\pm}_{p})=0$ i.e. frequency estimates are unbiased.
  \item $\EE((\hat{h}^{\pm}_{p}-h^{\pm}_{p})^2)\leq v^{\pm}$  i.e. $\mse$ of frequency estimator is bounded by $v^{\pm} = Cn^\pm\alpha$.
  \item For distinct $p,p'\in \{0,1\}^{\leq \alpha}$ with $|p|=|p'|$, $\hat{h}^{\pm}_p$ and $\hat{h}^{\pm}_{p'}$ are independent i.e. the frequency estimates are pairwise independent.
  \item For all $m\leq \log(d)$, the lists  $(\hat{h}^\pm_{p})_{p\in \{0,1\}^{\leq m}}$ and  $(\hat{h}^\pm_{p})_{p\in \{0,1\}^{> m}}$ are independent of each other. 
  \end{enumerate} Then, $\mse(\uaucest)$ is given by
  \begin{equation*}
    Cn^-n^+\alpha^2\Big(2n+(4a+1)\min(n^-,n^+)+\frac{21
    \sqrt{2nC\alpha}}{\sqrt{a}-1}\Big)
  \end{equation*}
\end{theorem}




\paragraph{Instantiating $\hat{h}$.}
So far, Theorem~\ref{thm:aucacc} does not yield a complete algorithm as it
does not specify an algorithm for computing estimates $\hat{h}$ of a
hierarchical
histogram that satisfy the conditions of Theorem~ \ref{thm:aucacc}. In
Section~\ref{sec:instantiating-h} of the supplementary, we show how to
instantiate such an algorithm in a communication-efficient manner by
combining ideas from \cite{HeavyHitters:17}, in particular the use of the
Hadamard transform, with an modified version of the protocol from 
\cite{KCS:2019}. This leads
to the following result.

\begin{theorem}
\label{thm:mseorder}
  There is a one-round non-interactive
  protocol for computing AUC in the local model
  with $\mse$ bounded by $O(\alpha^2\log
(1/\delta)/n\epsilon^2)$ under $(\epsilon,\delta)$-LDP and $O
(\alpha^3/n\epsilon^2)$ under $\epsilon$-LDP.
  Every user submits one bit, and the server does
  $O(nlog(d))$ computation and requires $O(log(d))$ additional reconstruction space.
\end{theorem}

%% file: subfiles/generic_2pc.tex

\if\arxiv1
\section{Generic Protocols from 2PC}
\else
\section{GENERIC PROTOCOLS FROM 2PC}
\fi
\label{sec:generic_2pc}

So far, we have proposed a specialized LDP protocol for AUC, and a generic LDP
protocol which requires some assumption on the kernel function or the data
distribution to guarantee nontrivial error bounds.
 We conjecture that no LDP protocol can guarantee nontrivial
error for arbitrary kernels and distributions, but we leave this as an
open question for future work.

In this section, we slightly relax the model of LDP by allowing pairs
of users $i$ and $j$ to compute a randomized version $\tilde{f}(x_i,x_j)$ of
their kernel value $f(x_i,x_j)$ with 2-party secure computation (2PC). This gives
rise to a computational differential privacy (CDP) model \cite{Mironov2009}.
Unsurprisingly, we show that in this model we can
match the MSE of $O(\frac{\ln(1/\delta)}{n\epsilon^2})$ for computing regular 
(univariate) averages in the $
(\epsilon,\delta)$-LDP model by using advanced composition results 
\citep{composition}. However, such a protocol requires $O(n^2)$ communication
as all pairs of users need to compute $\tilde{f}(x_i,x_j)$ via 2PC, and does
not satisfy pure $\epsilon$-DP.

\textbf{Proposed protocol.} To address these limitations, we propose that the aggregator
asks only a (random) subset of pairs of users $(i,j)$ to submit their
randomized
kernel value $\tilde{f}(x_i,x_j)$. The idea is to trade-off between the
error due to privacy (which increases as more pairs are used, due to budget
splitting) and the \emph{subsampling error} (for not averaging over all pairs).
Given a positive integer $P$ (which
should be thought of as a small
constant independent of $n$) and assuming $n$ to be even for simplicity, we
propose the following protocol:
\begin{enumerate}
\item \emph{Subsampling:} The aggregator samples $P$ independent permutations
$\sigma_1,\dots,\sigma_P\in\mathfrak{S}_n$ of the set of users $
\{1,\dots,n\}$. This defines a set of $Pn/2$ pairs $\mathcal{P}=(\sigma_p
(2i-1, 2i))_{p\in[P], 1 \leq i \leq N/2}$.
\item \emph{Perturbation:} For each pair of users $(i,j)\in\mathcal{P}$,
users compute $\tilde{f}(x_i,x_j)$ via 2PC and sends it to
the aggregator.
\item \emph{Aggregation:} The aggregator computes an estimate of $U_f$ as a
function of $\{\tilde{f}(x_i,x_j)\}_{(i,j)\in\mathcal{P}}.$
\end{enumerate}

\textbf{Analysis.} We have the following result for the Laplace mechanism
applied to real-valued
kernel functions (the extension to randomized response for discrete-valued
kernels is straightforward). The proof relies on an exact characterization
of the subsampling error by leveraging results on the variance of incomplete
$U$-statistics \citep{Blom76}, see Section~\ref{sec:proof-2pc} of the supplementary for details.

\begin{theorem}[2PC subsampling protocol with Laplace mechanism]
\label{thm:subsampling}
Let $\epsilon>0$, $P\geq 1$ and assume that the kernel $f$ has values in
$[0,1]$.
Consider our subsampling protocol above with $\tilde{f}(x_i,x_j)=f
(x_i,x_j) + \eta_{ij}$ where $\eta_{ij}\sim Lap(P/\epsilon)$, and the
estimate computed as $\widehat{U}_{f,n}=\frac{2}{Pn}\sum_{(i,j)\in\mathcal{P}}
\tilde{f}(x_i,x_j)$. Then the protocol satisfies $\epsilon$-CDP, has a total
communication cost of $O(Pn)$ and
\if\arxiv0
\begin{align*}
\label{eq:error_2pc}
\mse(\widehat{U}_{f,n}) =~& \frac{2}{Pn} \Big( 2(P-1)\big(1-\frac{1}
{n-1}\big)\zeta_1
\\&+ \big(1 + \frac{P-1}{n-1}\big)\zeta_2 \Big) + \frac{2P}
{n\epsilon^2},
\end{align*}
\else
\begin{align*}
\label{eq:error_2pc}
\mse(\widehat{U}_{f,n}) = \frac{2}{Pn} \Big( 2(P-1)\big(1-\frac{1}
{n-1}\big)\zeta_1
+ \big(1 + \frac{P-1}{n-1}\big)\zeta_2 \Big) + \frac{2P}
{n\epsilon^2},
\end{align*}
\fi
where $\zeta_1$ and $\zeta_2$ are defined as in \eqref{eq:ustat-variance}.
\end{theorem}
The MSE in Theorem~\ref{thm:subsampling} is of $O(\frac{1}{Pn} + \frac{P}
{n\epsilon^2})$.
Remarkably, this shows that the $O(1/n)$ variance
of the estimate that uses all pairs is preserved when subsampling only $O(n)$
pairs. This is made possible by the strong dependence structure in the
$O(n^2)$ terms of the original $U$-statistic. As expected, $P$ rules a
trade-off
between the errors due to subsampling and to privacy: the larger
$P$, the smaller the former but the larger the latter (as each user
must split its budget across $P$ pairs). The optimal value of $P$ depends on
the kernel function and the data distribution (through $\zeta_1$ and
$\zeta_2$) on the one hand, and the privacy budget $\epsilon$ on the other
hand. This trade-off, along with the optimality of the proposed
subsampling schemes, are discussed in more details Section~\ref{sec:proof-2pc}
of the supplementary material. In practice and as illustrated in our experiments,
$P$ can be set to a small constant.

\textbf{Implementing 2PC.}
Securely computing the randomized kernel value $\tilde{f}(x_i,x_j)$
can be done efficiently for many kernel functions and local randomizers of
interest, as the number of parties involved is limited to $2$.
We assume semi-honest parties \citep[see][for a definition of this threat
  model]{goldreich}. A suitable 2PC technique
in this application are garbled circuits \citep{yao,yaoproof,mpc-pragmatic}, which are
well-suited to compute Boolean comparisons as required in several of the
kernels mentioned in Section~\ref{sec:ustat-ex}. The circuits for computing
the kernels can then be extended with output perturbation following ideas
from~\cite{DBLP:conf/eurocrypt/DworkKMMN06} and~\cite{DBLP:journals/iacr/ChampionSU19}.
We refer to Section~\ref{sec:implement-2PC} of the supplement for details on design and complexity.







\begin{remark}[Beyond 2PC]
One could further relax the model to allow multi-party secure computation with
more than two parties, e.g. by extending the garbled circuit computing the
kernel with secure aggregation over the $Pn$ pairs before performing output
perturbation. This would recover the utility of centralized DP at the cost of
much more computation and quadratic communication, which is not practical,
as well as robustness. More interesting trade-offs may be achieved by
securely aggregating small subsets of pairs.
We leave the careful analysis of such extensions to future work.
\end{remark}

%% file: subfiles/exp.tex

\if\arxiv1
\section{Experiments}
\else
\section{EXPERIMENTS}
\fi
\label{sec:experiments}

\if\arxiv1
  \def\figsize{.7}
\else
  \def\figsize{.9}
\fi

\begin{figure}
    \centering
    \includegraphics[width=\figsize\columnwidth]
    {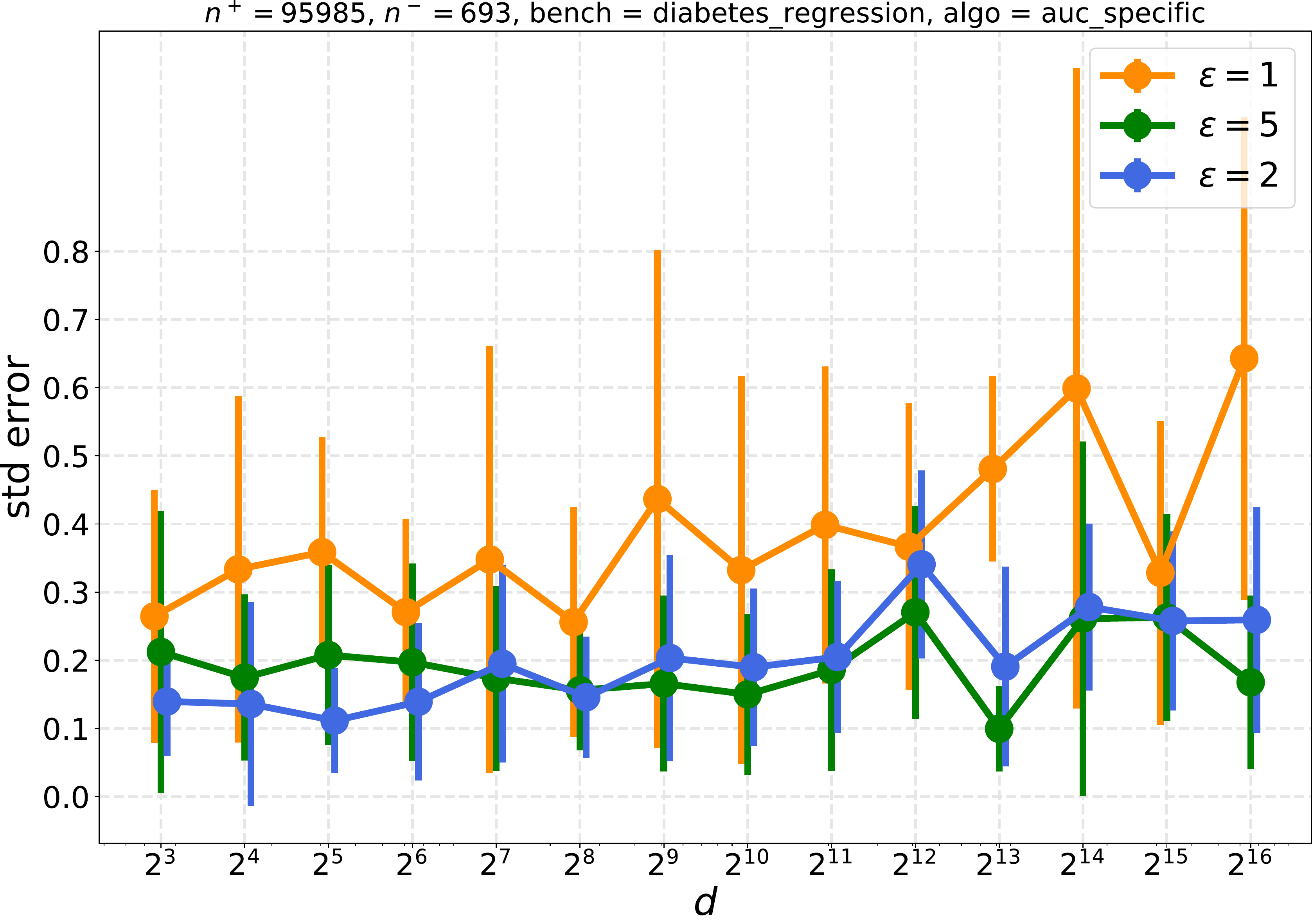}
    \caption{Mean and std. dev. (over 20 runs) of the
    absolute error of our AUC protocol on the scores of a logistic regression
    model trained on a Diabetes dataset.}
    \label{fig:experiments-auc}
\end{figure}

\textbf{AUC.}
We use the Diabetes dataset~\citep{diabetes-data} for the binary classification
task of determining whether a patient will be readmitted in the next $30$ days
after being discharged. We train a logistic regression model $s$ which is used
to score data points in $[0, 1]$, and apply our protocol to privately compute
AUC on the test set. Patients readmitted before $30$ days form the positive
class. Class sizes
are shown in Figure~\ref{fig:experiments-auc}. Class information is not
considered sensitive, as opposed to the score $s(x)$ on private user data
$s(x)$, which includes detailed medical information. Figure~\ref{fig:experiments-auc} shows the standard error achieved by our protocol
for different values of the domain size $d$. For each value of $d$ we run our
protocol with inputs $s(\texttt{fp}(s(x_i), d))$, where $\texttt{fp}$ denotes
a discretization into the domain $[0..d-1]$. The plot shows that the
protocol is quite robust to the choice of $d$, and that increasing $\epsilon$
beyond $2$ does not improve
results significantly. Recall that the error of our AUC protocol depends on
the size of the smallest class, which is quite small here (only $693$
examples).

\if\arxiv1
  \def\figsize{.7}
\else
  \def\figsize{.9}
\fi

\begin{figure}[t]
  \centering
     \includegraphics[width=\figsize\columnwidth]
     {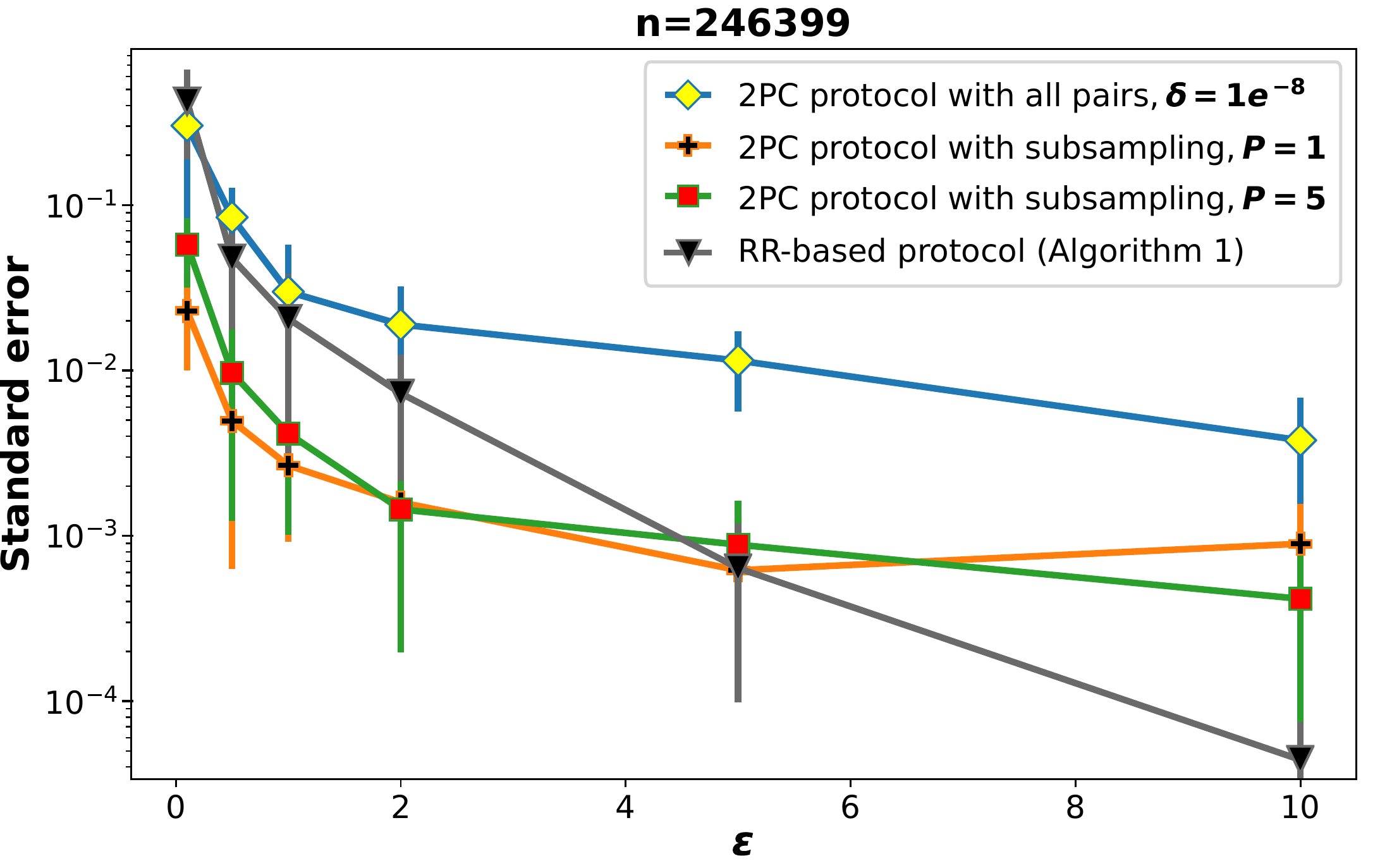}
  \caption{Mean and standard deviation (over 20 runs) of the absolute
  error in KTC on Tripadvisor dataset.} \label{fig:2pc}
\end{figure}

\textbf{Kendall's tau.} We use the Tripadvisor dataset~\citep{tripadvisor:10}. The dataset consists of discrete user ratings (from
scale -1
to 5) for hotels in San Francisco over many service quality metrics such as
room service, location, room cleanliness, front desk service etc. After
discarding the records with missing values, we have over 246K records.
Let $x_i=(y_i,z_i)$ be ratings given by user $i$ to the room ($y_i$) and the
cleanliness ($z_i$).
The goal is to privately estimate the Kendall's tau coefficient (KTC) between
these two variables, whose true value is $0.58$.
We compare the privacy-utility trade-off of our randomized
response protocol (Algorithm~\ref{algo:DP-hist} without quantization, since
inputs can take only $6\times 6=36$ values), our 2PC protocol based on
subsampling, and the 2PC protocol that computes all pairs and relies on
advanced composition (for which we set $\delta=1e^{-8}$). The 2PC primitive to
compute $\tilde{f}(x_i,x_j)$ is simulated.
The results shown in Figure~\ref{fig:2pc} show that the 2PC protocol with all
pairs performs worst due to composition. 
The randomized response protocol performs slightly better, thanks to the
small domain size. Finally, our 2PC protocol with subsampling achieves the
lowest error by roughly an order of magnitude in high privacy
regimes ($\epsilon\leq 2$) while keeping the communication cost linear
in $n$. As predicted by our analysis, $P=1$ is best in
high privacy regimes, where the error due
to privacy dominates the subsampling error. We also see that $P>1$ can
be used to reduce the overall error in low privacy regimes (for $\epsilon=10$
one can use an even larger $P$ to match the error of the randomized response
protocol).


%% file: subfiles/conclu.tex

\if\arxiv1
\section{Concluding Remarks}
\else
\section{CONCLUDING REMARKS}
\fi
\label{sec:conclu}

In this paper, we tackled the problem of computing $U$-statistics from private
user data, covering many statistical quantities of broad interest which were
not addressed by previous private protocols.

\textbf{Relative merits of our protocols.} Our three protocols are
largely complementary, insofar as each of them is well-suited to specific
situations. Our first protocol (quantization followed by randomized response)
can gracefully handle cases where the kernel is Lipschitz or the data is
discrete in a small domain. It may also work well for non-Lipschitz kernels
when quantizing data into few bins does not lose too much information
(e.g., when the data distribution is smooth enough). As the latter hypothesis
is difficult to assess in advance, we argue that there can exist specialized
protocols that work well for non-Lipschitz kernels on continuous data or large
discrete domains. Our second protocol illustrates this for the case of AUC: we
leverage a hierarchical histogram structure to scale much better with the
number of bins than the first protocol (see Section~\ref{sec:more-experiments} for experiments
comparing these two protocols). Finally, if one is
willing to slightly
relax the LDP model to allow pairwise communication among users, and if the
kernel can be computed efficiently via 2PC, our third protocol is expected
to perform best in terms of accuracy.

\textbf{Extension to higher degrees.} While we focused on \emph{pairwise}
$U$-statistics, our ideas can be extended to higher degrees. A prominent
example is the Volume Under the ROC
Surface, the generalization of the AUC to multi-partite ranking 
\citep{CRV13}.

\textbf{Future work.} A promising direction for future work is to develop private
multi-party algorithms for learning with pairwise losses 
\citep{pairwise-losses,CBC2016} by combining private stochastic gradient
descent for standard empirical risk minimization 
\citep{Bassily2014a,Shokri2015} and our protocols to compute the gradient
estimates.


%% file: subfiles/appendix.tex

\section{Details and Proofs for Generic LDP Protocol}

We start by introducing some notations. We denote by $k$ the number of bins
of the quantization and by $\pi:\mathcal{X}\rightarrow[k]$ the
quantization scheme such that for any data point $x\in\mathcal{X}$, $\pi(x)$
denotes the quantized version of $x$ (i.e., its image under $\pi$). Let $e_i$
denote the vector of length
$k$ with a one in the $i$-th position and $0$ elsewhere.
A kernel function $f_A$ on the quantized domain $[k]$ is
fully described by a matrix $A\in\mathbb{R}^{k\times k}$ such that $f_A(i,j)
= A_{i,j} = e_i^TAe_j
$. We denote by $U_{A,\pi}=\E_{\mu\times\mu}[A_{\pi
(x),\pi(y)})]$ the quantized analogue to the quantity $U_f$.

The proposed protocol, described in Algorithm~\ref{algo:DP-hist}, applies
generalized randomized response on data quantized with $\pi$ and uses this
to compute an unbiased estimate of a quantized $U$-statistic. The choice of
the quantized kernel $A$ will be discussed below.
Crucially, there are two sources of error in this protocol. More precisely, the
mean squared error of the estimate $\widehat{U}_{f,n}$ returned by Algorithm~
\ref{algo:DP-hist} can be bounded as follows:
\begin{equation}
\label{eq:error_decompo}
\mse(\widehat{U}_{f,n}) \leq (U_f - U_{A,\pi})^2 + \E[(U_{A,\pi} - \widehat{U}_
{f,n})^2].
\end{equation}
The first term corresponds to the error due to
quantization, while the second one is the estimation error due to
randomization needed to satisfy local differential privacy.
The latter will increase with $k$, thereby constraining $k$ to remain
reasonably small.
We will thus need to rely on assumptions on either the kernel or the data
distribution to be able to control the error due to quantization.

In line with the error decomposition in \eqref{eq:error_decompo}, we conduct
our analysis by considering the effect of sampling and randomization together.
Therefore, we will not provide a direct bound on
the error between our estimate and the U-statistic of the sample, but directly
with respect to the population quantity $U_f$.
We now show how to control the two sources of error, which are easily combined
to yield Theorem~\ref{thm:dphist-main}.

\subsection{Bounding the Error of Randomized Response on Discrete Domain}
\label{sec:rr-discrete}

In this part, we bound the second term in \eqref{eq:error_decompo}: we
consider that the data is discrete ($\mathcal{X}=[k]$) and derive error bounds
for the estimate $\widehat{U}_{f,n}$ with respect to $U_{A,\pi}$ for a given kernel
 function $\tilde{f}(i,j)=A_{i,j}$. 
We propose to use the generalized randomized response mechanism as our local
randomizer $\RRR$.
We introduce some notations. Let $\beta$ be the probability of $\RRR$
selecting a response uniformly at random, i.e. let $\PP(\RRR(x)=y)=\beta/k+
(1-\beta)\chi_{x=y}$, let $b$ be the vector of length $k$ with every entry $\beta/k$.
For convenience, we denote the data sample by $x_1,\dots,x_n\in[k]$. Note that
these data points are drawn i.i.d. from a distribution $D$ over $[k]$ 
(which follows from $\mu$ and $\pi$) such that $\PP
(x_i=j)=D_j$.

With these notations, we write the expected value of the
discretized kernel computed directly on the randomized data points:
\begin{align*}
  \EE[(e_{\RRR(x_1)})^TAe_
  {\RRR(x_2)}]&=\EE(((1-\beta)e_{x_1}+b)^TA((1-\beta)e_
  {x_2}+b)) \\
  &=\EE((1-\beta)^2e_{x_1}^TAe_{x_2}+(1-\beta)(e_{x_1}+e_{x_2})^TAb+b^TAb).
\end{align*}
This is a biased estimator of $f_A(x_1,x_2)=e_
{x_1}Ae_
{x_2}$ due to the effect of the randomization. We correct for this by
adding terms and scaling, leading to the estimator used in
Algorithm~\ref{algo:DP-hist}:
\begin{equation*}
  \widehat{f}_A(\RRR(x_1),\RRR(x_2))=(1-\beta)^{-2}(e_{\RRR(x_1)}-b)^TA(e_
  {\RRR
  (x_2)}-b).
\end{equation*}
This is an unbiased estimator of the population U-statistic, as for fixed
$x_1$ and $x_2$ it is an unbiased estimator of $f_A(x_1,x_2)$. Averaging over
all pairs of randomized inputs, we get the proposed estimator:
\begin{equation*}
  \widehat{U}_{f,n}=\binom{n}{2}^{-1}\sum_{1\leq i<j\leq n}\widehat{f}_A(\RRR
  (x_i),\RRR(x_j)),
\end{equation*}
which is itself a U-statistic on the randomized sample.
As this estimator is unbiased, its mean squared error is equal to its
variance, for which the following lemma gives an exact expression.
\begin{lemma}
\label{lem:discretevariance}
The variance of $\widehat{U}_{f,n}$ is given by
  \begin{equation*}
  \binom{n}{2}^{-1}\left(\frac{2n-3}{(1-\beta)^{2}}\var(e_
  {(\RRR(x_1)})^TAD)+\frac{1}{(1-\beta)^{4}}\EE(\var((e_{\RRR(x_1)}-b)Ae_{\RRR
  (x_2)}\mid\RRR(x_1)))\right)
  \end{equation*}
\end{lemma}

\begin{proof}
$\widehat{U}_{f,n}$ is a $U$-statistic, hence its variance is given by 
\eqref{eq:ustat-variance} where $\zeta_1=\var(\EE(
\widehat{f}_A(\RRR
(x_1),\RRR
(x_2))\mid \RRR(x_1)))$ and $\zeta_2=\var(\widehat{f}_A(\RRR(x_1),\RRR
(x_2)))$.
We first simplify $\zeta_1$:
\begin{align*}
  \zeta_1&=(1-\beta)^{-4}\var(\EE((e_{\RRR(x_1)}-b)^TA(e_{\RRR(x_2)}-b)\mid
  \RRR
  (x_1))) \\
&=(1-\beta)^{-4}\var((e_{\RRR(x_1)}-b)^TA((1-\beta)D)) \\
&=(1-\beta)^{-2}\var((e_{\RRR(x_1)}-b)^TAD) \\
&=(1-\beta)^{-2}\var((e_{\RRR(x_1)})^TAD).
\end{align*}
Similarly for $\zeta_2$, we have:
\begin{align*}
  \zeta_2&=\var(\EE(\widehat{f}(\RRR(x_1),\RRR(x_2))\mid \RRR(x_1)))+\EE(\var(\widehat{f}(\RRR(x_1),\RRR(x_2))\mid \RRR(x_1))) \\
&=\zeta_1+(1-\beta)^{-4}\EE(\var((e_{\RRR(x_1)}-b)^TA(e_{\RRR(x_2)}-b)\mid
\RRR(x_1))) \\
&=\zeta_1+(1-\beta)^{-4}\EE(\var((e_{\RRR(x_1)}-b)^TAe_{\RRR(x_2)}\mid \RRR
(x_1))).
\end{align*}
Substituting the values of $\zeta_1$ and $\zeta_2$ in the variance expression
gives the result.
\end{proof}

Assuming a uniform bound on the values of $f$ allows a clear and
simple bound on the variance.

\begin{corollary}
  If $f(x,x')\in [0,1]$ for all $x,x'$, then
  \begin{equation*}
    \var(\widehat{U}_{f,n})\leq \frac{1}{n(1-\beta)^2}+\frac{(1+\beta)^2}{2n(n-1)(1-\beta)^4}.
  \end{equation*}
\end{corollary}

\begin{proof}
  Under the boundedness of $f$, the random variable
  $(e^{\RRR(x_1)})^TAD$ takes values in $[0,1]$ and so has variance at most
  $1/4$, whilst the random variable $(e^{\RRR(x_1)}-B)^TAe^{\RRR(x_2)}$ takes
  values in $[-\beta,1]$ and so has variance at most $(1+\beta)^2/4$. Substituting these into Lemma \ref{lem:discretevariance} gives the result.
\end{proof}

To achieve local differential privacy with parameter $\epsilon$, $\beta$
should be taken to be $k/(k+e^\epsilon-1)$. This leads directly to the
following result.

\begin{corollary}[Variance under randomized reponse]
Let $\mathcal{X}=[k]$ and assume $f$ takes values in
$[0,1]$. We have:
\begin{equation*}
  \var(\widehat{U}_{f,n}) \approx \frac{(1+k/\epsilon)^2}{n}+\frac{(1+k/\epsilon)^4}
  {2n^2} \approx \frac{k^2}{n\epsilon^2},
\end{equation*}
where the approximation holds for small $\epsilon$ and $n\gg k^2/\epsilon^2$.
\end{corollary}

The above result shows that for fixed $\epsilon$ the error incurred by
this estimator is within a constant factor of the error due to the finite
sample setting. As expected, $k$ should be reasonably small for the protocol
to yield any utility.




\subsection{Bounding the Error of Quantization}
\label{sec:discretization}


We now study the effect of quantization, which is needed to control the
error due to privacy when the domain is continuous or has large cardinality.
Recall that we quantize $\mathcal{X}$ using a projection $\pi:
\mathcal{X}\rightarrow
[k]$ (assumed to be simple rounding for simplicity), and our goal is to
approximate $U_f=\EE_{\mu\times\mu}(f(x,y))$ by
$U_{A,\pi}=\EE_
{\mu\times\mu}(A_{\pi(x),\pi(y)})$, which we can privately estimate using
the results of Section~\ref{sec:rr-discrete}.

The error incurred by the quantization can be written as follows:
\begin{align}
\nonumber
  (U_f - U_{A,\pi})^2
  &\leq\int \int (f(x,y)-A_{\pi(x),\pi(y)})^2 d\mu(y)d\mu(x) \\
\label{eq:errorsum}
&  =\sum_{i,j=1}^k\int_{\pi^{-1}(i)}\int_{\pi^{-1}(j)} (f(x,y)-A_{i,j})^2 d\mu
(y)d\mu(x).
\end{align}

To bound this quantization error, we need additional assumptions. We
consider two options, each suggesting a different choice for the quantized
kernel $A_{i,j}$.
We first consider a Lipschitz assumption on the original kernel function, for
which the preferred quantized kernel minimizes the worst-case bound on 
\eqref{eq:errorsum}.
Then, we consider a smoothness assumption on the data distribution, leading to
a quantized kernel that attempts to minimize the average-case error.
We stress the fact that in some cases these quantized statistics will match
or at least be very close, meaning that the particular choice of quantized
kernel will not be crucial.

\subsubsection{Assumption of Lipschitz Kernel Function}

Our first assumption is motived by the fact that for all data
distributions $\mu$, the quantization error \eqref{eq:errorsum} can be
bounded by 
\begin{equation}
  \label{eq:errorstatistic}
  \sum_{i,j=1}^k\mu(\pi^{-1}(i))\mu(\pi^{-1}(j)) \max_{\substack{x\in\pi^{-1}
  (i)\\y\in\pi^{-1}(j)}}(f(x,y)-A_{i,j})^2.
\end{equation}
This bound is minimized by choosing the quantized kernel to be
\begin{equation}
\label{eq:midpoint}
  A^{Mid}_{i,j}=\frac{1}{2}\max_{\substack{x\in\pi^{-1}(i)\\y\in\pi^{-1}(j)}}f(x,y)+\frac{1}{2}\min_{\substack{x\in\pi^{-1}(i)\\y\in\pi^{-1}(j)}}f(x,y),
\end{equation}
which we will call the midpoint kernel. With this kernel we can define
 \begin{equation*}
   \Delta_{i,j}=\frac{1}{4}\Big(\max_{x\in\pi^{-1}(i),y\in\pi^{-1}
   (j)}f(x,y)-\min_{x\in\pi^{-1}(i),y\in\pi^{-1}(j)}f(x,y)\Big)^2,
 \end{equation*}
which allows us to write the bound in equation \ref{eq:errorstatistic} as
\begin{equation*}
  \Delta=\sum_{i,j=1}^k\mu(\pi^{-1}(i))\mu(\pi^{-1}(j)) \Delta_{i,j}.
\end{equation*}
Note that this is itself a discrete $U$-statistic over the population. The
error can now be bounded through a bound on $\Delta$.
A natural way to achieve this is to uniformly bound $\Delta_{i,j}$, which can
be done by assuming that the kernel function $f$ is Lipschitz. This allows to
control the error within each bin.

\begin{lemma}[Quantization error for Lipschitz kernel functions]
\label{lem:lip}
Let $\mathcal{X}=[0,1]$ and assume $f:\mathcal{X}\times
\mathcal{X}\rightarrow\mathbb{R}$ is $L_f$-Lipschitz in each input. Let the
set of bins be $\{(2l-1)/2k:l\in[k]\}$ and let the quantization scheme
$\pi$ perform simple rounding of inputs (affecting them to the nearest bin).
Then we
have $(U_f - U_{A,\pi})^2\leq L_f^2/2k^2$.
\end{lemma}
\begin{proof}
By the Lipschitz property of $f$, we have that $|f(x,y)-f(x',y')|\leq L_f
(|x-x'|+|y-y'|)$ for all $x,x',y,y'$. Since the diameter of each bin is equal
to $1/k$, we have $\Delta_{i,j}\leq L_f^2/2k^2$ for all $i,j\in[k]$ and the
lemma follows.
\end{proof}

As desired, the quantization error decreases with $k$. Note that the
Lipschitz assumption is met in the important case of the Gini
mean difference, while it does not hold for AUC and Kendall's tau. Bounding
$\Delta$ is not the right approach for such kernels: indeed, for AUC, $\Delta_
{i,i}=1/2$ and so for data distributions $\mu$ with $\mu(\pi^
{-1}(i))=1$ for some $i$, the quantization error $\Delta\geq 1/2$.
In the next section, we consider generic kernel functions under a
smoothness assumption on the data distributions.

%


\begin{remark}[Empirical Estimation of $\Delta$]
The quantization error $\Delta$ is a discrete $U$-statistic which can be
estimated from the data
collected to estimate $U$. This provides a good empirical estimate of $\Delta$
after the fact. However, as the data has to be collected before the estimate
can be made it provides no guidance in choosing $\pi$  (this might be
addressed by a multi-round protocol, which we leave for future work). The
empirical assessment of $\Delta$ 
may provide a tighter bound on the actual error than can be
ascertained by the worst-case Lipschitz assumption. \end{remark}

\subsubsection{Assumption of Smooth Data Distribution}

We now consider a smoothness assumption on the data distribution $\mu$.
Specifically, we assume that the density $d\mu/d\lambda$ with respect to a
measure $\lambda$ (which varies little on $\pi^{-1}(i)$ for all $i$) is
$C$-Lipschitz.

In this case, a more sensible choice of quantized kernel is given by
\begin{equation}
\label{eq:A-avg}
  A^{Avg}_{i,j}=\frac{1}{\lambda(\pi^{-1}(i))\lambda(\pi^{-1}(j))}\int_{\pi^
  {-1}(i)}\int_{\pi^{-1}(j)} f(x,y)d\lambda(y)d\lambda(x),
\end{equation}
which we call the average kernel as the value of $A^{Avg}_{i,j}$
corresponds to the (normalized) expectation of $f(x,y)$, with respect to $\lambda$,  over points $x$ and
$y$ that are mapped to bin $i$ and $j$ respectively.

Under our smoothness assumption, the quantization error \eqref{eq:errorsum}
can be bounded as follows.

\begin{lemma}[Quantization error for smooth distributions]
Let $\mathcal{X}=[0,1]$ and $f(x,y)\in [0,1]$ for all $x,y$.\footnote{Similar arguments can be made in more general metric spaces.} Assume
that $d\mu/d\lambda$ is $L_\mu$-Lipschitz.
Then we have $(U_f - U_{A,\pi})^2\leq 4L_\mu^2 D^2 (1 + L_\mu^2 D^2)$,
where
$D$ is the maximum diameter of the quantization bins.
\end{lemma}
\begin{proof}
For notational convenience, let us denote $
\bar{\mu_i}:=\mu(\pi^{-1}(i))$ and $
\bar{\lambda_i}:=\lambda(\pi^{-1}(i))$ for each $i$. The absolute
quantization error with quantized kernel \eqref{eq:A-avg} is given by:
\begin{align}
  &\left| \sum_{i,j=1}^k\int_{\pi^{-1}(i)}\int_{\pi^{-1}(j)} f(x,y)-A^{Avg}_
  {i,j} d\mu(y)d\mu(x) \right|\nonumber \\
  \leq& \sum_{i,j=1}^k\left|\int_{\pi^{-1}(i)}\int_{\pi^{-1}(j)} f(x,y)-A^{Avg}_{i,j} d\mu(y)d\mu(x) \right|\nonumber \\
  =& \sum_{i,j=1}^k\left|\int_{\pi^{-1}(i)}\int_{\pi^{-1}(j)} f(x,y)d\mu
  (y)d\mu(x) -\bar{\mu_i}\bar{\mu_j}A^{Avg}_{i,j}  \right|
  \label{eq:smooth1}
\end{align}

Note that
\begin{align*}
  \int_{\pi^{-1}(i)}\int_{\pi^{-1}(j)} f(x,y) d\mu(y)d\mu(x) =\int_{\pi^{-1}
  (i)}\int_{\pi^{-1}(j)} f(x,y) \frac{d\mu(y)}{d\lambda(y)}\frac{d\mu(x)}{d\lambda(x)} d\lambda(y)d\lambda(x),
\end{align*}
and
\begin{align*}
  \bar{\mu_i}\bar{\mu_j}A^{Avg}_{i,j}
  = \frac{\bar{\mu_i}\bar{\mu_j}}{\bar{\lambda_i}\bar{\lambda_j}}\int_{\pi^{-1}(i)}\int_{\pi^{-1}(j)} f(x,y) d\lambda(y)d\lambda(x).
\end{align*}
Plugging these equations into \eqref{eq:smooth1} we get:
\begin{align}
  &\left|\int_{\pi^{-1}(i)}\int_{\pi^{-1}(j)} f(x,y) d\mu(y)d\mu(x)-\bar{\mu_i}\bar{\mu_j}A^{Avg}_{i,j}\right|\nonumber \\
  =&\left|\int_{\pi^{-1}(i)}\int_{\pi^{-1}(j)} f(x,y) \Big(\frac{d\mu(y)}
  {d\lambda
  (y)}\frac{d\mu(x)}{d\lambda(x)}-\frac{\bar{\mu_i}\bar{\mu_j}}{
  \bar{\lambda_i}\bar{\lambda_j}}\Big) d\lambda(y)d\lambda(x)\right|\nonumber
  \\
\leq&\int_{\pi^{-1}(i)}\int_{\pi^{-1}(j)} \left|\frac{d\mu(y)}{d\lambda(y)}
\frac{d\mu(x)}{d\lambda(x)}-\frac{\bar{\mu_i}\bar{\mu_j}}{\bar{\lambda_i}
\bar{\lambda_j}}\right| d\lambda(y)d\lambda(x) \\
  \leq&\int_{\pi^{-1}(i)}\int_{\pi^{-1}(j)} L_\mu D\Big(\max_{z\in
  \pi^
  {-1}(i)}d\mu(z)/d\lambda(z)+\max_{w\in \pi^{-1}(j)}d\mu(w)d\lambda(w)\Big)
  d\lambda(y)d\lambda(x) \label{eq:smooth2}\\
\leq& \int_{\pi^{-1}(i)}\int_{\pi^{-1}(j)} (L_\mu D) d\lambda(y)d\mu
(x) +\int_{\pi^{-1}(i)}\int_{\pi^{-1}(j)} (L_\mu D) d\mu(y)d\lambda(x)
\nonumber
\\
&+\int_{\pi^{-1}(i)}\int_{\pi^{-1}(j)} (L_\mu D) (2L_\mu D)d\lambda(y)d\lambda(x)\label{eq:smooth3}.
\end{align}
Summing over all $i,j$ and taking the square finally gives the result:
\begin{align}
   &\left(\int\int f(x,y) d\mu(y)d\mu(x)-\sum \mu(\pi^{-1}(i))\mu(\pi^{-1}
   (j))A^{Avg}_{i,j}\right)^2\nonumber \\
  \leq& 4L_\mu^2 D^2+4L_\mu^4D^4.\nonumber
\end{align}
\end{proof}

The diameter of quantization bins is typically of order $1/k$, hence the
quantization error is of order $1/k^2$.
In practice, $\lambda$ can simply be taken to be
Lebesgue measure, hence computing \eqref{eq:A-avg} amounts to averaging the
kernel function over all possible points $(x,y)\in\mathcal{X}$ that fall in
the bins $(i,j)$, and can be easily approximated by Monte Carlo sampling when
one does not have a closed form expression for the integral.

\section{Details and Proofs for AUC Protocol}

\subsection{Proof of Theorem~\ref{thm:aucacc}}
\label{sec:auc-proof}

We define $\recursed^m=\{p\in \{0,1\}^m:\forall p'\preceq p, \tilde{h}^-_{p'}\tilde{h}^+_{p'}>\tau\}$ as the set of nodes recursed on at level $m$. Similarly, and for $m>0$, let $\activated^m=\recursed^{m-1}\cdot \{0,1\}$ be the active nodes at level $m$, i.e. those to be either recursed on or discarded. Then, the set of discarded nodes at level $m$ is defined as $\discarded^m = \activated^m\setminus\recursed^m$. Our algorithm has two main sources of error: (i) the one incurred on by discarded nodes, i.e. nodes in $\bigcup\limits_{i\in[\alpha]}\discarded^m$ for whose intervals the algorithm uses a rough estimate, and (ii) the error in the estimating the contribution to the UAUC of the recursed nodes, i.e. nodes in $\bigcup\limits_{i\in[\alpha]}\recursed^m$.

The threshold $\tau$ is carefully chosen according to the error of the estimator $\hat{h}$ to balance these two errors. In this way we translate error bounds for $\hat{h}^+_p, \hat{h}^-_p$ into error bounds for $\uaucest$. Our proof starts by bounding the expected size of $\recursed^m$.

\begin{lemma}
\label{lem:recursed}
  Consider the instantiation of Equation~\ref{eq:uauc-privacy} with a frequency oracle for estimating $h^\pm$ satisfying $\forall p\in\{0,1\}^{\leq\alpha}: \big(\EE(\hat{h}^\pm_p) = h^\pm_p , {\EE((\hat{h}^\pm_p - h^\pm_p)^2)}\leq v^\pm\big)$, with $v^\pm = C n^\pm\alpha$, i.e. the estimate is unbiased and has uniformly bounded $\mse$. If $a>1$, then for all $m \in [\alpha]$, 
  \begin{equation*}
    \EE(|\recursed^m|)\leq \frac{\sqrt{n^+}+\sqrt{n^-}}{2(\sqrt{a}-1)\sqrt{C\alpha}}\leq \frac{1}{\sqrt{a}-1}\sqrt{\frac{n}{2C\alpha}}
  \end{equation*}
\end{lemma}
\begin{proof}
  Let $\hat{n}^\pm=\sum_{p\in A^m}\max(\hat{h}^{\pm}_p,0)$, the sum of the positive estimated counts of active nodes at level $m$.

  Note that if $p\in \recursed^m$ then $\tilde{h}^+_p\tilde{h}^-_p\geq a\sqrt{v^+v^-}$. In this case either, $\hat{h}^\pm_p=\tilde{h}^\pm_p$ and thus $\hat{h}^+_p\hat{h}^-_p\geq a\sqrt{v^+v^-}$, $\hat{h}^-_p\neq\tilde{h}^-_p$ and thus $\hat{h}^+_p>2 \sqrt{av^+}$, or $\hat{h}^-_p\neq\tilde{h}^-_p$ and thus $\hat{h}^-_p> 2 \sqrt{av^-}$. In any of these cases $\hat{h}^+_p/\sqrt{av^+}+\hat{h}^-_p/\sqrt{av^-}\geq 2$.

  Therefore 
  \begin{equation*}
    2|\recursed^m|\leq \sum_{p\in \recursed^m} \frac{\hat{h}^+_p}{\sqrt{av^+}}+\frac{\hat{h}^-_p}{\sqrt{av^-}} \leq \frac{\hat{n}^+}{\sqrt{av^+}}+\frac{\hat{n}^-}{\sqrt{av^-}}
  \end{equation*}
  and thus
  \begin{equation*}
    \EE(|\recursed^m|)\leq \frac{\EE(\hat{n}^+)}{2\sqrt{av^+}}+\frac{\EE(\hat{n}^-)}{2\sqrt{av^-}}.
  \end{equation*}
  We bound $\EE(\hat{n}^\pm)$ as follows
  \begin{align*}
    \EE(\hat{n}^\pm)=&\sum_{p\in A^m}\EE(\max(\hat{h}^{\pm}_p,0)) 
    \leq n^{\pm}+\sum_{p\in A^m}\EE(\max(e^{\pm}_p,0)) \\
    \leq&n^{\pm}+\EE(|A^m|)\max_{p\in A^m}\EE(\max(e^{\pm}_p,0)) 
    \leq n^{\pm}+\EE(|\recursed^{m-1}|)\max_{p\in A^m}\EE(|e^{\pm}_p|) \\
    \leq&n^{\pm}+\EE(|\recursed^{m-1}|)\max_{p\in A^m}\sqrt{\EE(|e^{\pm}_p|^2)} 
    \leq n^{\pm}+\EE(|\recursed^{m-1}|)\sqrt{v_{\pm}}
  \end{align*}

  We can now use this to bound the expression for $\EE(|\recursed^m|)$.
  \begin{equation}
    \label{eq:inductionstep}
     \EE(|\recursed^m|)\leq \frac{n^+}{2\sqrt{av^+}}+\frac{n^-}{2\sqrt{av^-}}+ \frac{\EE(|\recursed^{m-1}|)}{\sqrt{a}}
  \end{equation}

  We now need a bound on $\EE(|\recursed^{m-1}|)$ so we will proceed by induction.

  Let $B=\frac{\sqrt{n^+}+\sqrt{n^-}}{2(\sqrt{a}-1)\sqrt{C\alpha}}$. We take $\EE(|\recursed^{m-1}|)\leq B$ as the induction hypothesis, and $\EE(|\recursed^0|)=1\leq B$ as the base case.

  The expression on the right hand side of inequality \ref{eq:inductionstep} is a monotonically increasing function of $\EE(|\recursed^{m-1}|)$ and has a fixed point

  \begin{align*}
    \frac{\frac{n^+}{2\sqrt{av^+}}+\frac{n^-}{2\sqrt{av^-}}}{1- \frac{1}{\sqrt{a}}}=&\frac{\frac{n^+}{2\sqrt{v^+}}+\frac{n^-}{2\sqrt{v^-}}}{\sqrt{a}-1} 
    =\frac{\sqrt{\frac{n^+}{C\alpha}}+\sqrt{\frac{n^-}{C\alpha}}}{2(\sqrt{a}-1)} 
    = \frac{\sqrt{n^+}+\sqrt{n^-}}{2(\sqrt{a}-1)\sqrt{C\alpha}} 
    = B.
  \end{align*}

  Thus we can conclude that
  \begin{equation}
    \label{eq:lemconc}
    \EE(|\recursed^m|)\leq B
  \end{equation}
  completing the induction and thus \eqref{eq:lemconc} holds for all $m$.

  Finally we note that
  \begin{equation*}
    \sqrt{n^+}+\sqrt{n^-}\leq \sqrt{2n}
  \end{equation*}
  and so
  \begin{equation*}
    \frac{\sqrt{n^+}+\sqrt{n^-}}{2(\sqrt{a}-1)\sqrt{C\alpha}}\leq \frac{1}{\sqrt{a}-1}\sqrt{\frac{n}{2C\alpha}}
  \end{equation*}
  completing the proof.
\end{proof}

\input{subfiles/accuracyproof_arxiv}




\subsection{Instantiating the Private Hierarchical Histogram \texorpdfstring{$\hat{h}$}{h}}
\label{sec:instantiating-h}

Theorem~\ref{thm:aucacc} does not yield a complete algorithm as it only
states that, if we had a differentially private algorithm for computing
estimates of a hierarchical histogram that satisfy the conditions of Theorem~
\ref{thm:aucacc}, then we could solve AUC with the stated accuracy. In this
section we instantiate such algorithm and show that, besides the required error guarantees, our proposal also has other nice properties, namely (i) it is one round, (ii) each user sends a single bit, and (iii) it is sublinear in $d$ processing space at the server.

\subsubsection{Frequency Oracle}

Relevant previous work on estimating hierarchical histograms in the local model includes the work of~\cite{HeavyHitters:17}. While in that work the target problem is heavy hitters, their algorithm is similar to ours, as the server retrieves the heavy hitters by exploring a hierarchical histogram. Moreover their protocol -- called $\texttt{TreeHist}$ -- has the nice properties listed above, as it is one round, every user sends a single bit and requires reconstruction space sublinear in $d$. This satisfies the three above conditions. It is thus tempting to reuse the hierarchical histogram construction from \cite{HeavyHitters:17}. However, it does not satisfy the conditions of Theorem~\ref{thm:aucacc}, as it is not guaranteed to be unbiased.

Alternative recent algorithms for constructing hierarchical histograms in the
local model are presented in~\cite{KCS:2019}, with the motivation of answering
range queries over a large domain. This proposal is much closer to what we need.
However, it has some shortcomings: first, although
it is one round, each user sends $O(\log(d))$ bits, and more importantly, it
requires space $O(d)$ space at the server, as it reconstructs the whole
hierarchical histogram. However, one can tweak the protocol from~
\cite{KCS:2019} to overcome these limitations. We shall first split the users
into $\log(d)$ groups (one for each level) and then for each level we shall
apply the frequency oracle. Algorithm~\ref{algo:local} and
Algorithm~\ref{algo:fo} show the local randomizer (user side) and
frequency oracle (server side) for each histogram.

\begin{algorithm2e}[t]
  \DontPrintSemicolon
  \LinesNumbered
  \SetKwComment{Comment}{{\scriptsize$\triangleright$\ }}{}
  \caption{Local Randomizer}\label{algo:local}
  {\bf Public Parameters:}~ Domain size $2^l$, privacy budget $\epsilon$.\\
  \KwIn{Private index $q$}
  \KwOut{A single bit $z$ submitted to the server}
  \BlankLine
  $j\leftarrow [0..2^l-1]$~~~~~~~~\Comment{Selected uniformly at random}
  $y := \frac{1}{\sqrt{2^l}}(-1)^{\langle j, q\rangle}$~~~~~~~\Comment{$y$ is $M_{j,x^l}$, where $M\in \{-1, 1\}^{2^l\times 2^l}$ is a Hadamard matrix}
  $z := \begin{cases}
    y & \text{with probability } \frac{e^\epsilon}{1+e^\epsilon}\\
    \neg y & \text{otherwise}\end{cases}$~~~~~~~~\Comment{Submit randomized response on $y$}
  Send $j,z$ to the Aggregator
\end{algorithm2e}

\begin{algorithm2e}[t]
  \DontPrintSemicolon
  \LinesNumbered
  \SetKwComment{Comment}{{\scriptsize$\triangleright$\ }}{}
  \caption{Frequency Oracle}\label{algo:fo}
  {\bf Public Parameters:}~ Domain size $2^l$, privacy budget $\epsilon$.\\
  \KwIn{The index $j_i$ and response $z_i$ of each party $i$ and an index $q$ to estimate the frequency of}
  \KwOut{$z$ an estimated count of $q$}
  \BlankLine
  For all $i$, $y_i := \frac{1}{\sqrt{2^l}}(-1)^{\langle j_i, q\rangle}$~~~~~~~\Comment{$y_i$ is $M_{j_i,q}$, where $M\in \{-1, 1\}^{2^l\times 2^l}$ is a Hadamard matrix}
  $z_q := \frac{e^\epsilon+1}{e^\epsilon-1}\sum\limits_{i} y_iz_i$~~~~~~~~\Comment{De-bias the sum of contributions}
  Return $z_q$
\end{algorithm2e}

Let $\text{count}(q)$ be the true count of an index $q$ in a histogram. The following lemma is shown in \cite{KCS:2019}.

\begin{lemma}
  \label{lem:fomse}
  The frequency oracle, Algorithm \ref{algo:fo}, run with $n_l$ users is unbiased $\EE(z_q)=\text{count}(q)$ and satisfies the following bound on the \mse.
  \begin{equation}
    \EE((z_q-\text{count}(q))^2)\leq \frac{4n_le^\epsilon}{(e^\epsilon-1)^2}
  \end{equation}
\end{lemma}

Additionally we require the following lemma on the frequency oracle satisfies
condition (3) in Theorem \ref{thm:aucacc} which is given by the following
lemma.

\begin{lemma}
  \label{lem:pairwiseindep}
   For distinct $q,q'\in[0..2^l-1]$, $z_q$ and $z_{q'}$ are independent i.e. the responses of the oracle are pairwise independent.
\end{lemma}

\begin{proof}
  As each user is independent of every other user it suffices to show that each user's contribution to the two entries are independent. Suppose that a user has input $q''\neq q'$, chooses index $j$ to report and let $b$ be a bit indicating that the user chose $z=\neg y$ in Algorithm \ref{algo:local}. That user's contributions to the two estimates (scaled by $2^l$) are $(-1)^{\langle j, q\rangle+\langle j, q''\rangle+b}$ and $(-1)^{\langle j, q'\rangle+\langle j, q''\rangle+b}$. Note that we can consider $j,q,q'$ and $q''$ as elements of $\mathbb{F}_2^l$. Then $q+q''$ and $q'+q''$ are distinct and $q'+q''\neq 0$. These two facts imply respectively that $\langle j, q'+q''\rangle$ is independent of $\langle j, q+q''\rangle$ and that $\langle j, q'+q''\rangle$ is uniformly distributed in $\mathbb{F}_2$. Thus the contributions are independent.
\end{proof}

\subsubsection{Splitting Strategies}

We will instantiate $\hat{h}$ by running the frequency oracle above for each
level of the hierarchy. The main choice remaining is how to determine which
users contribute to each layer, we will consider two possibilities here.
Firstly we can have everyone contribute to all layers, splitting their privacy
budget. Alternatively, users can be split evenly across levels at random, each
contributing to only one frequency oracle. Another possibility is to assign
each user to a level independently and uniformly, this is similar to splitting them evenly though adds slightly more noise and is more complicated to analyse. In all cases, conditions 1 and 3 in Theorem \ref{thm:aucacc} follow from Lemmas \ref{lem:fomse} and \ref{lem:pairwiseindep}.

\paragraph{Splitting Privacy Budget Across Layers.}

In the case of everyone contributing to all layers the privacy budget can be
split using either basic or advanced composition. In either case condition 4
from Theorem \ref{thm:aucacc} holds as the randomness for each layer is independent.

For pure differential privacy we must use basic composition. This allows us to
run each frequency oracle can be run with a privacy budget of $
\tilde{\epsilon}=\epsilon/\alpha$. Lemma \ref{lem:fomse} then gives a bound of
$O_\epsilon(n\alpha^2)$ on the mean squared error of each entry. While this is
insufficient to establish condition 2 of Theorem \ref{thm:aucacc}, similar
arguments can be used to prove that the algorithm built in
this way achieves pure differential privacy at the cost of an
$\alpha$ factor in the \mse.

If we instead settle for $(\epsilon,\delta)$-differential privacy, and assume
for convenience that $\epsilon\leq\sqrt{\alpha}\ln(2)$, advanced composition
allows each frequency oracle to be run with privacy budget $
\tilde{\epsilon}=\epsilon/(\sqrt{\alpha}(1+\sqrt{2\log(1/\delta)}))$.
Condition
2 in Theorem \ref{thm:aucacc} then holds for some $C$ depending on $\epsilon$
and $\delta$.
This is the implementation and analysis that gives Theorem \ref{thm:aucacc} as it is stated.

\paragraph{Splitting Users Across Layers.}

When splitting users across levels the frequency oracles can each be run with
privacy budget $\epsilon$. However, each oracle will have only $n/\alpha$
users and there is a subsampling error between the total sample and the
input given to the frequency oracle. The squared error due to subsampling is
$O
(1/n)$ thus Lemma \ref{lem:fomse} provides a $O_\epsilon(n\alpha)$ bound on the \mse. This means that condition 2 of Theorem \ref{thm:aucacc} holds. This would provide a version of Theorem \ref{thm:aucacc} with pure differential privacy, however condition 4 from Theorem \ref{thm:aucacc} fails to hold. Intuitively this is
because if a user contributes to one level they can't contribute to another level. There are still two things that can be proved about this version of the algorithm.

Firstly, it is still possible to prove a result like Theorem \ref{thm:aucacc},
but in which the $\mse$ is $\alpha$ times bigger. The proof of this result
follows the same steps as that of Theorem \ref{thm:aucacc} except
that the martingale difference sequences argument must be replaced by a bound not assuming pairwise independence.

A second way of viewing this algorithm is to think of each input as being
drawn independently from some population distribution and then compare the
output to the AUC of that distribution. That is, given a pair of distributions
$\mathcal{D}^\pm$, $\mathcal{S}^{\pm}$ is obtained by sampling each value
independently from $
\mathcal{D}^\pm$. Denote $\auc_\textrm{pop} = \E_{x^+\sim\mathcal{D}^+,x^-\sim
\mathcal{D}^-}{}[f(x^+,x^-)]$ and let $\mse_
\textrm{pop}(\aucest)=\E[(\auc_\textrm{pop} - \aucest)^2]$.
The
fact that each of the users has an independent identically distributed input
means that the contribution to each layer is independent, i.e. we can recover Theorem \ref{thm:aucacc} with $\mse$ replaced by $\mse_\textrm{pop}$.
This alternative notion of $\mse_\textrm{pop}$ is the correct notion to work
with if the purpose of the deployment of the algorithm is to find the AUC of
 the population the sample is drawn from rather than just of the sample. This is likely to be the case in many applications. 

\paragraph{Summary.} Table~\ref{tab:auc_errors} summarizes the choices in the
algorithm and analysis. The resulting
orders of the $\mse$, corresponding to Theorem~\ref{thm:mseorder} are given
in the final column.

\begin{table}
\centering
  \begin{tabular}{ccl}
    \toprule 
    Splitting & Analysis & Error in $\aucest$ \\ \midrule
    Privacy budget & Basic composition & $\mse\leq O(\frac{\alpha^3}
    {n\epsilon^2})$  \\
    Privacy budget & Advanced composition & $\mse\leq O(\frac{\alpha^2\log
    (1/\delta)}{n\epsilon^2})$  \\ 
    Users & w.r.t. sample & $\mse\leq O(\frac{\alpha^3}{n\epsilon^2})$  \\
    Users & w.r.t. population & $\mse_{\textrm{pop}}\leq O(\frac{\alpha^2}
    {n\epsilon^2})$  \\
    \bottomrule
  \end{tabular}
  \caption{Summary of error bounds for our AUC protocol for different
  splitting strategies and analysis techniques.}
  \label{tab:auc_errors}
  \end{table}

\section{Details and Proofs for 2PC Protocol}

\subsection{Proof of Theorem~\ref{thm:subsampling} and Discussion}
\label{sec:proof-2pc}

\begin{proof}
The $\epsilon$-DP follows from the Laplace mechanism and the simple
composition property of DP (observing that each input $x_i$ appears in exactly
$P$ pairs in $\mathcal{P}$).

It is easy to see that $\widehat{U}_{f,n}$ is unbiased, hence we only need
to bound its variance. We will separate the part due to
subsampling and
the part due to privacy. To this end, we decompose $\widehat{U}_{f,n}$ into a
noise-free
term and a noisy term:
\begin{equation}
\widehat{U}_{f,n} = \underbrace{\frac{2}{Pn}\sum_{(i,j)\in\mathcal{P}}f
(x_i,x_j)}_{\widehat{U}_{f,n,P}}
+
\frac{2}{Pn}\sum_{(i,j)\in\mathcal{P}}\eta_{ij}.
\end{equation}
The noisy term is an average of independent Laplace random variables: its
variance is equal to $2P/n\epsilon^2$.

The quantity $\widehat{U}_{f,n,P}$ is known as an incomplete
$U$-statistic, whose variance is given by \cite{Blom76}:
\begin{equation}
\label{eq:var_incomp}
\var(\widehat{U}_{f,n,P}) = \frac{4}{(Pn)^2}\Big(\mathbb{E}[f_1(
\mathcal{P})]\zeta_1 + \mathbb{E}
[f_2(\mathcal{P})]\zeta_2\Big)
\end{equation}
where $\zeta_1=\var(f(x_1,X_2)\mid x_1))$, $\zeta_2=\var(f
(X_1,X_2)$, and $f_1(\mathcal{P}),f_2(\mathcal{P})$ are the number of members of $
\mathcal{P}\times \mathcal{P}$ which have exactly 1 (respectively 2) indices
in common.

We first consider $\mathbb{E}[f_2(\mathcal{P})]$.
Recall that $\mathcal{P}$ is
constructed from $P$ permutations $\sigma_1,\dots,\sigma_P$ of the
set $\{1,\dots,n\}$. As each index appears exactly once in each permutation,
it suffices to consider the self pairs within permutations and the overlaps
across pairs of permutations we get:
\begin{align*}
\mathbb{E}[f_2(\mathcal{P})] &= \sum_{i<j}\sum_{p=1}^P\mathbb{E}\Big[q^p_{ij}(
\mathcal{P}) \big(1 + \sum_{p'\neq p}q^{p'}_{ij}(
\mathcal{P})\big)\Big],
\end{align*}
where $q^p_{ij}(\mathcal{P})$ is the number of pairs from the permutation $p$
that contain $\{i,j\}$. The probability of a pair $
(i,j)$ to appear in a given permutation is $1/(n-1)$, hence using the
independence between permutations we obtain:
\begin{align*}
\mathbb{E}[f_2(\mathcal{P})] &= \frac{n(n-1)}{2}\frac{P}{n-1}\Big(1 +
\frac{P-1}{n-1}\Big) = \frac{Pn}{2}\Big(1 + \frac{P-1}{n-1}\Big).
\end{align*}

For $\mathbb{E}[f_1(\mathcal{P})]$, using a similar reasoning we only have to
consider overlaps across each pair of permutations, in which each index pair
shares exactly one index with a single pair of
another permutation, except when the pair appears twice, hence:
\begin{align*}
\mathbb{E}[f_1(\mathcal{P})] &= \sum_{(i,j)\in\mathcal{P}} 2 (P - 1) - 2\sum_
{i<j}\sum_{p=1}^P\sum_{p'\neq p}
\mathbb{E}\big[q^p_{ij}(
\mathcal{P})q^{p'}_{ij}(
\mathcal{P})\big],\\
&= P(P-1)n - n(n-1)\frac{P(P-1)}{(n-1)^2}\\
&= P(P-1)n\Big(1 - \frac{1}{n-1}\Big)
\end{align*}

Putting everything together into \eqref{eq:var_incomp} we get the desired result.
\end{proof}

\paragraph{Optimal value of $P$.} The optimal value of $P$ depends on the
kernel function, the data
distribution and the privacy budget. Roughly speaking, setting $P$ larger
than $1$ can be beneficial when $\zeta_2$ is large compared to
$1/\epsilon^2$. On the other hand, when
$\zeta_2 = 2\zeta_1$ (which is the minimum value of
$\zeta_2$, corresponding to the extreme case where the kernel can in fact be
rewritten as a sum of univariate functions \citep{Blom76}), $\var(\widehat{U}_
{f,n})$ simplifies to $
\frac{4\zeta_1}{n} + \frac{2P}{n\epsilon^2}$ and $P=1$ is optimal. In
practice and as illustrated in our experiments, $P$ should be set to a small
constant.

\paragraph{Optimality of subsampling schemes.}
The proposed subsampling strategy is simple to implement and leads to an
optimal variance, up to an additive term of $\frac{2}{Pn}\frac{P-1}{n-1} 
(\zeta_2 - 2\zeta_1)\geq0$, among unbiased approximations based on $Pn/2$
pairs. Note that this additive term is $0$ when $P=1$ or $\zeta_2=2\zeta_1$, and
is in general negligible compared to the dominating terms for small
enough $P$. Optimal variance could be achieved at the cost of a more involved
sampling scheme.\footnote{In addition to having each data point
appear the same number of times in $\mathcal{P}$, one must ensure that no
pair appears more than once.}
Alternatively, sampling schemes that can be run independently by each user
without global coordination (such as sampling $P/2$ other users uniformly at
random) lead to a slight increase in variance as users are not guaranteed to
appear evenly across the sampled pairs.

\subsection{Implementing 2PC}
\label{sec:implement-2PC}

MPC is a subfield of cryptography concerned with the general problem of computing on private distributed data in a way in which only the result of the computation is revealed to the parties, and nothing else. In this paper the number of parties is limited to $2$, and the function to be computed is
$\widetilde{f}(x, y)$. There are several protocols that allow to achieve this
goal, with different trade-offs in terms of security, round complexity, and
also differing on how the functionality $\widetilde{f}$ is represented. These
alternatives include Yao's garbled circuits~\citep{yao, yaoproof}, the GMW
protocol~\citep{gmw}, and the SPDZ protocol~\citep{spdz}, among others. As
some of the
functions $\widetilde{f}$ we are interested in involve comparisons (e.g., Gini mean difference and
AUC), a Boolean representation is more suitable, as it will lead to a smaller circuit. Moreover, a constant round protocol is preferred in our setting, as as users might have limited connectivity. For this reason we choose garbled circuits as our protocol, for which~\citep{mpc-pragmatic} give a detailed description including crucial practical optimizations. Moreover, we assume semi-honest adversaries in the sequel \citep[see][for a definition of this threat model]{goldreich}. 

\paragraph{Circuits for kernels.} We illustrate the
main ideas on Gini mean difference and AUC. As circuits for floating point
arithmetic are large, they are usually
avoided in MPC, to instead rely on fixed point encodings. Hence, we assume that the parties have agreed on a precision, and hence $x,y$ are integers encoded in two's complement.

For Gini mean difference we need our 2PC protocol to compute $\fgini(x,y) :=
|x - y|$. Let $z$ be $x - y$, let $z_{k-1},\ldots,z_0$ be the binary encoding
of $z$, where the bitwidth $k$ will be a constant such as $32$ or $64$ in
practice, and let $s = z_{k-1}\cdots z_{k-1}$ be the sign bit of $z$ replicated $k$ times. Then $\fgini(x,y)$ can be computed as $(z + s) \oplus s$ and, thanks to the free-XOR optimization of garbled circuits \citep[see][]{mpc-pragmatic}, the garble circuit evaluation requires only a subtraction and a summation, and thus is very efficient.

For AUC we need our 2PC protocol to compute $\fauc(x,y) := x < y$, which requires a single comparison and thus a small number of binary gates to be evaluated in a garbled circuit.

\paragraph{Circuits for local randomizers.} The above circuits need to be extended with output perturbation corresponding to the Laplace and randomized response mechanisms discussed above. An important observation when designing efficient circuits for these tasks is the well-known fact that a random bit with bias $1/p$, for any integer $p$, can be generated from only two uniform random bits suffice, in expectation. Generating a uniformly random bit is easy (and extremely cheap using garbled circuits) in the semi-honest model: each party simply generates a random bit, and then inside the circuit a random bit is reconstructed as the XOR of two bits. As XORs are for free in garbled circuits this computation is very efficient. The problem of implementing differentially private mechanisms in MPC was discussed by \cite{DBLP:conf/eurocrypt/DworkKMMN06}, where the authors present small circuits for sampling from an exponential distribution requiring only a $log(k)$ biased random bits, which can be constructed in parallel. Recently,~\cite{DBLP:journals/iacr/ChampionSU19} proposed optimized constructions for several well-known differentially private mechanisms (including the geometric and Laplace mechanisms), and empirically showed their concrete efficiency.

\section{Additional Experiments}
\label{sec:more-experiments}

\subsection{AUC Experiments on Synthetic Data}
\label{sec:synthetic-experiments} 

\begin{figure}[t]
    \centering
    \subfigure[\texttt{auc\_one}]{\includegraphics[width=.325\textwidth]
    {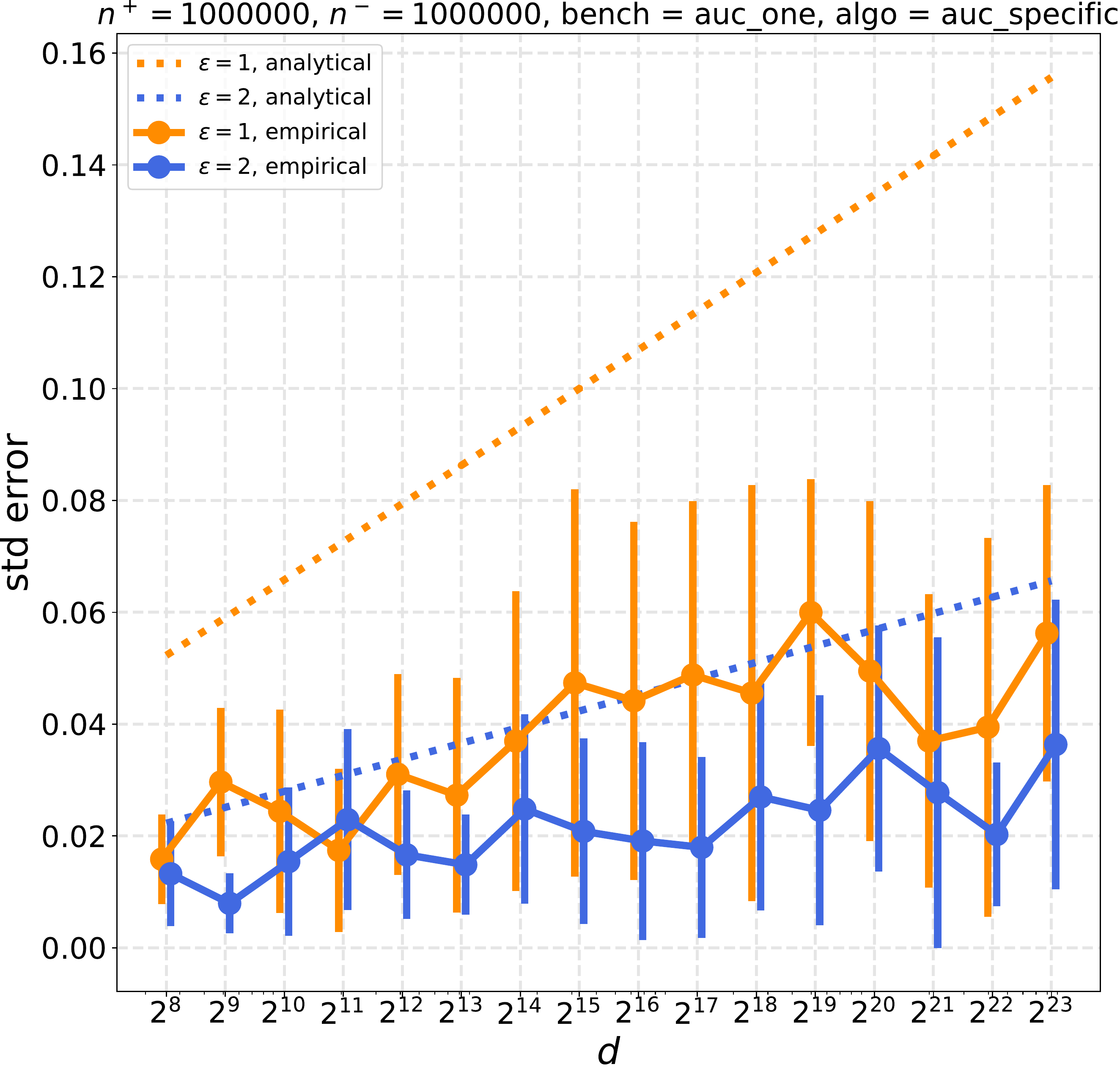}}
    \subfigure[\texttt{ur}]{\includegraphics[width=.32\textwidth]
    {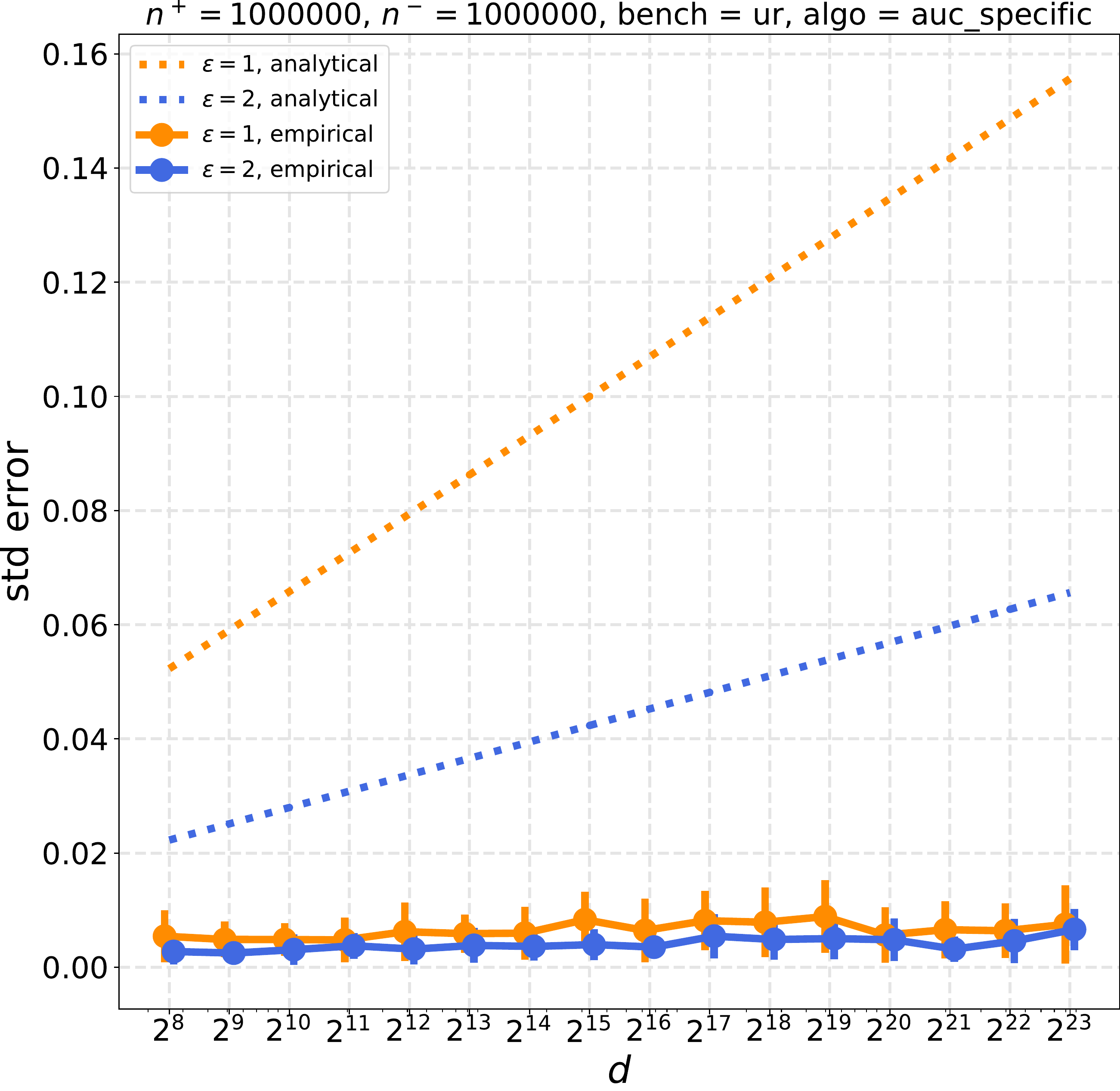}}
    \subfigure[\texttt{ithdigit}]{\includegraphics[width=.32\textwidth]
    {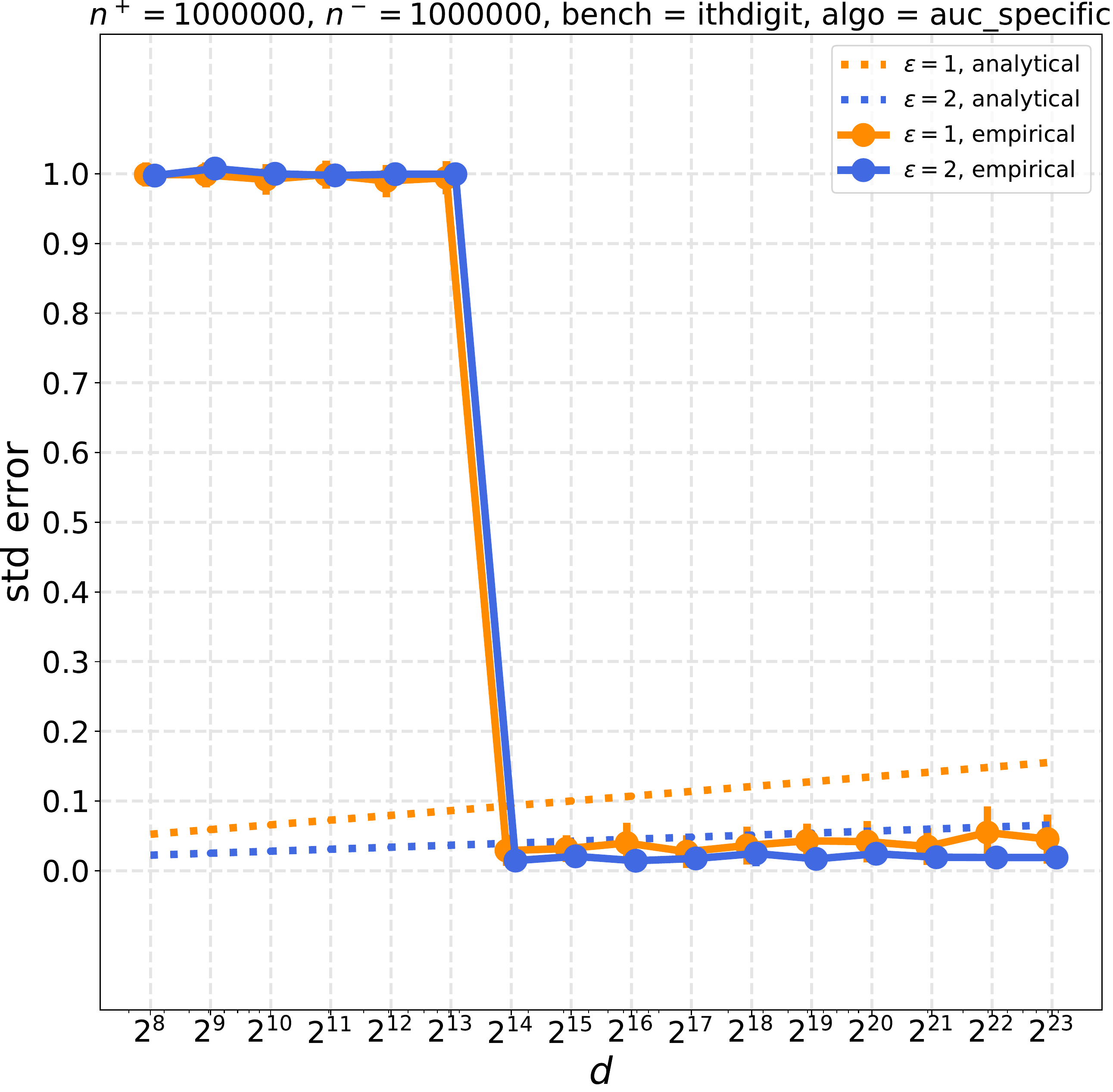}}
  \caption{Mean and std. dev. (over 20 runs) of the
    absolute error of our AUC-specific LDP protocol on three synthetic
  datasets.}
  \label{fig:experiments-auc-synthetic-auc}
\end{figure}

\begin{figure}[t]
    \centering
    \subfigure[\texttt{auc\_one}]{\includegraphics[width=.32\textwidth]
    {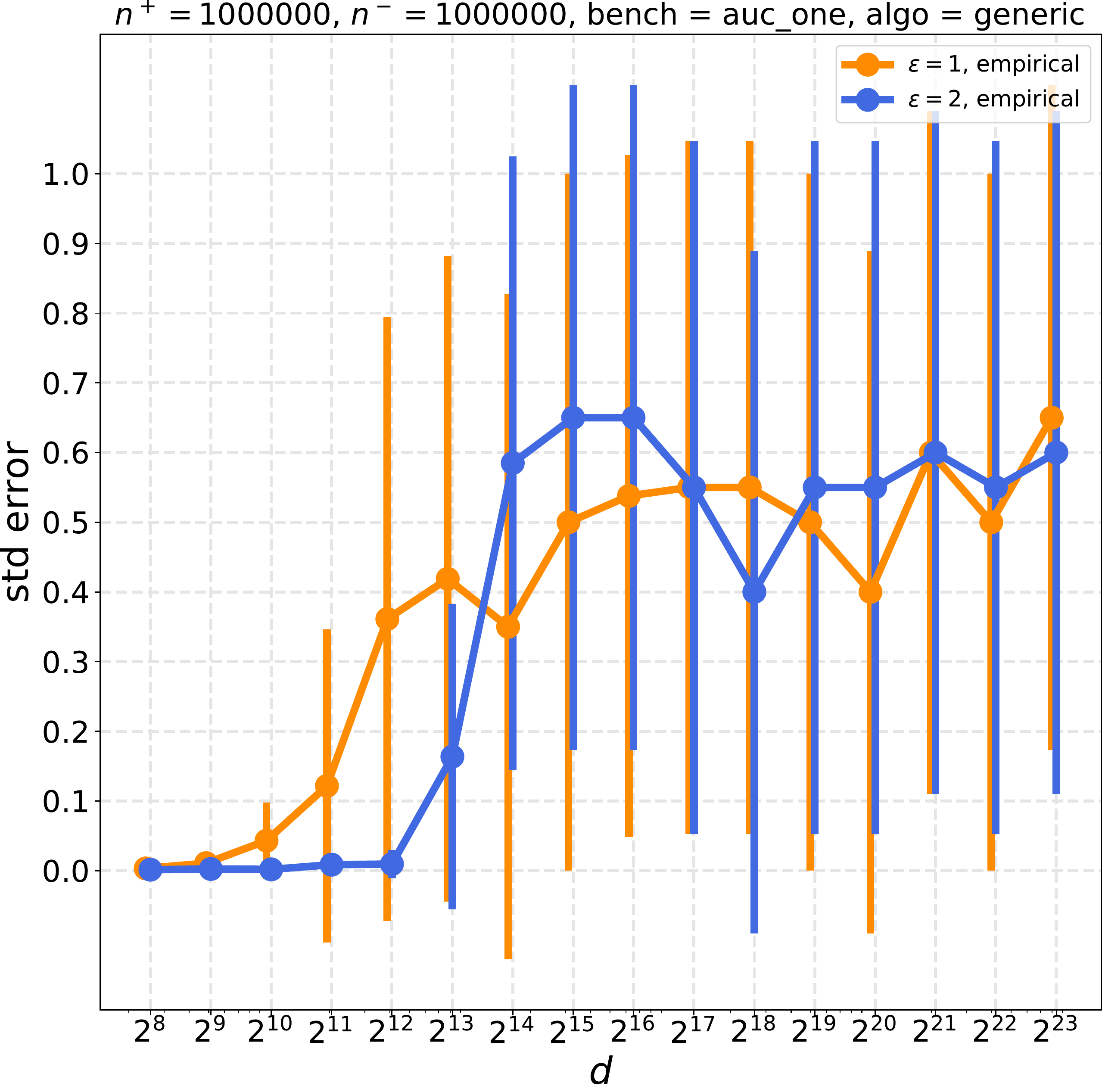}}
    \subfigure[\texttt{ur}]{\includegraphics[width=.32\textwidth]
    {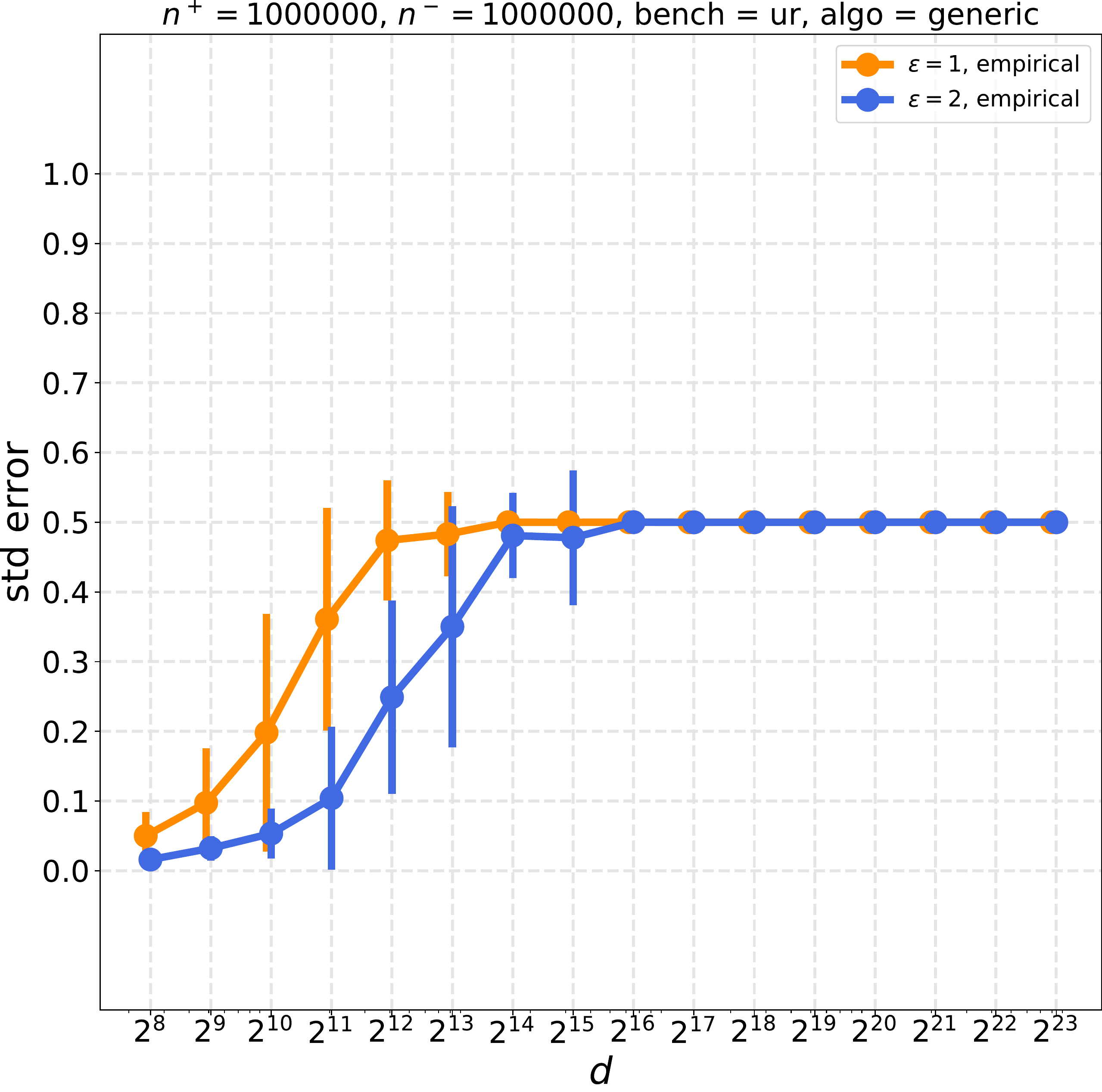}}
    \subfigure[\texttt{ithdigit}]{\includegraphics[width=.32\textwidth]
    {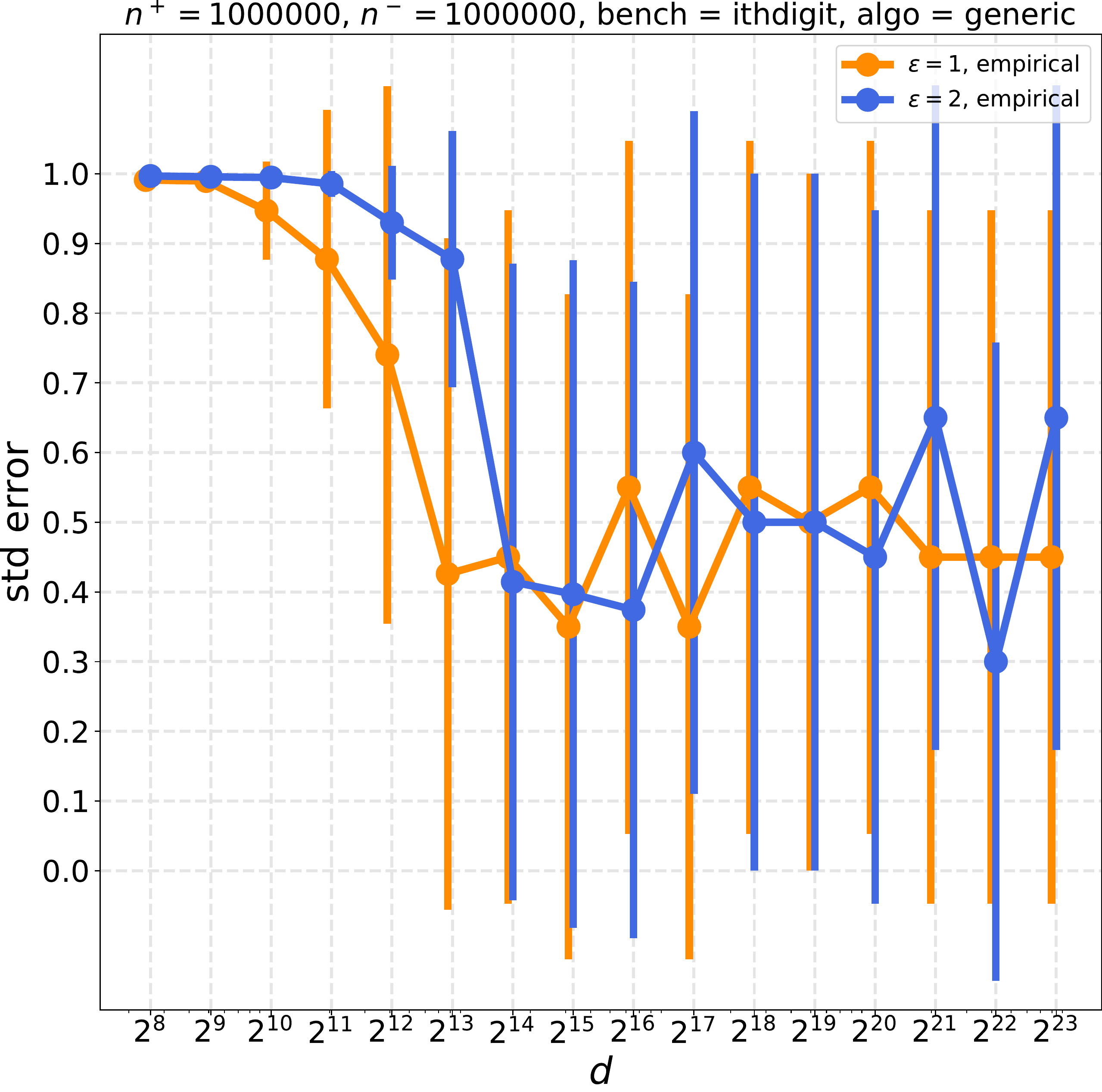}}
  \caption{Mean and std. dev. (over 20 runs) of the
    absolute error of our generic LDP protocol on three
  synthetic
  datasets.}
  \label{fig:experiments-auc-synthetic-generic}
\end{figure}

We illustrate the behavior of our AUC-specific LDP protocol of Section~\ref{sec:auc} on
three synthetic datasets, and compare its performance with that of our generic
LDP protocol of Section~\ref{sec:generic_ldp}. For all datasets, we have
$10^6$ inputs
in each class (positive and negative) but the score values of inputs are
distributed differently:
\begin{itemize}
\item \texttt{auc\_one} consists of two distinct inputs $
(d-1,1)$ and $(0,-1)$ each occurring $10^6$ times.
\item In \texttt{ur}, the score value of an input is drawn
independently and uniformly from $[0..d-1]$, regardless of its class.
\item \texttt{ithdigit} consists of two distinct inputs $(10^{-4},1)$ and $
(0, -1)$ each occurring $10^6$ times.
\end{itemize}

Figure~\ref{fig:experiments-auc-synthetic-auc}
shows the error that our AUC-specific protocol incurs on the three datasets.
On \texttt{auc\_one}, our AUC-specific protocol incurs considerable error due
to significant recursion error,
$E^R_m$, being incurred for every level.
This example illustrates that our
analysis for the AUC-specific protocol is not far from being tight.
On \texttt{ur}, the error is much lower. This is because (i) the algorithm
does not explore any of the lower sections of the tree and so no recursion error is incurred whilst exploring it, and (ii) within intervals that are discarded the points are uniformly distributed so the estimation of the AUC within that interval as a half is effective. Both of these effects will occur approximately whenever the data is smooth so one can expect the algorithm to do better in the case of smooth data than the analytic bounds indicate.
Finally, on \texttt{ithdigit}, the protocol does not learn anything when the
quantization is smaller or equal to $2^{13}$, which is expected as all inputs
are quantized to the same bin. However, we see that the protocol achieves low
error for $d\geq 2^{14}$. Importantly, in all cases our AUC-specific protocol
scales nicely with the size of the domain. This allows to be rather agnostic
about the level of quantization needed for the problem at hand, which is often
not known in advance.

In contrast, we can see on Figure~\ref{fig:experiments-auc-synthetic-generic}
that the generic protocol scales very poorly with the domain size due to the
use of randomized response. On \texttt{auc\_one} and \texttt{ur}, data can be
quantized to a small domain without losing much relevant information, leading
to good performance. In the case of \texttt{ithdigit} however, the generic
protocol incurs very large error in all regimes: quantizing to small domain
maps all inputs to the same bin, while quantizing to large domain leads to
large error due to privacy.

\subsection{Results of the Generic Protocol on Diabetes dataset}

\begin{figure}[t]
    \centering
    \includegraphics[width=.5\textwidth]
    {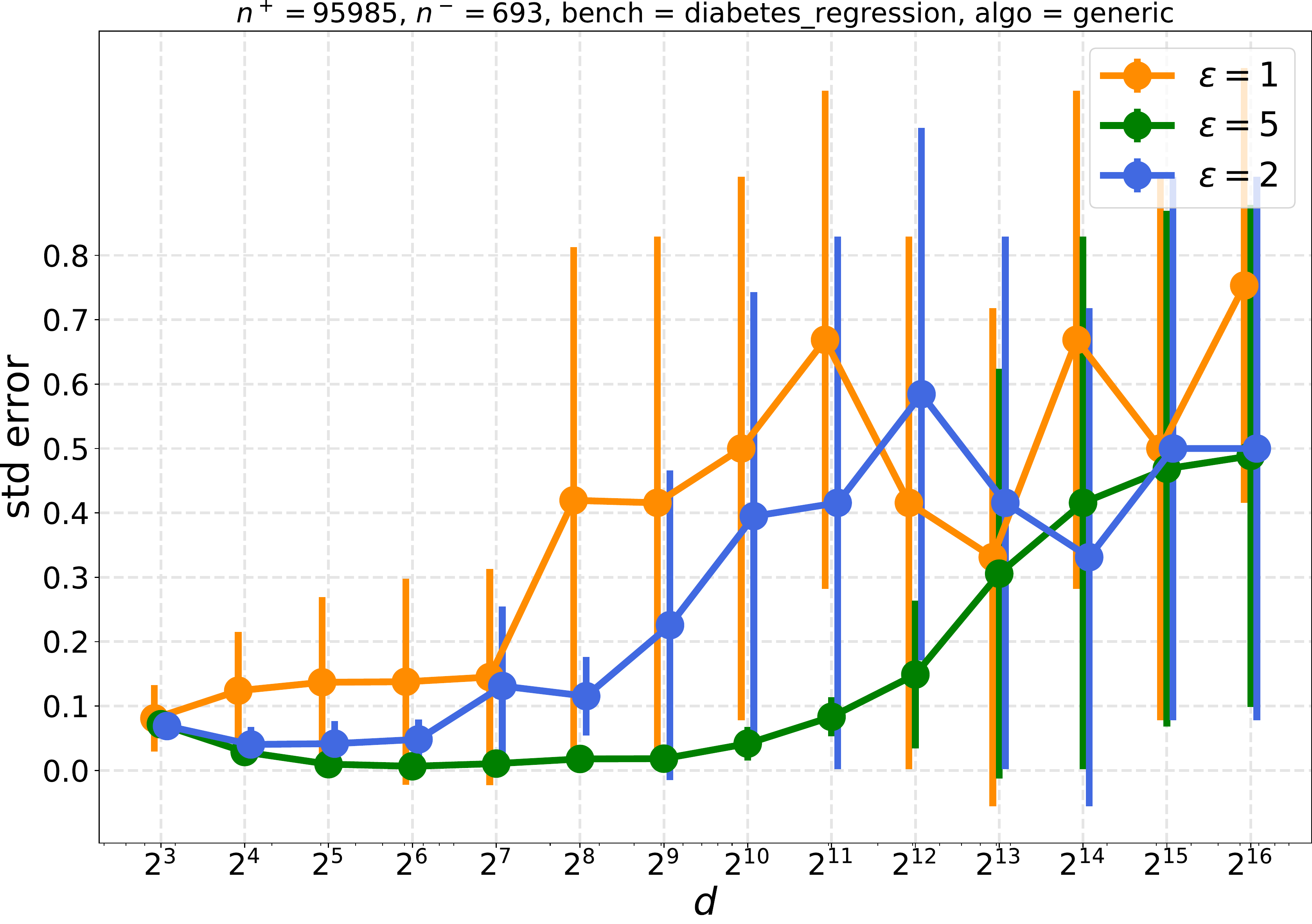}
  \caption{Mean and std. dev. (over 20 runs) of the
    absolute error of our generic LDP protocol on the scores of a logistic
    regression model trained on a Diabetes dataset.}
  \label{fig:experiments-auc-diabetes-generic}
\end{figure}

Figure~\ref{fig:experiments-auc-diabetes-generic} shows the results of
our generic LDP protocol of Section~
\ref{sec:generic_ldp} for the problem of computing the AUC on the Diabetes
dataset (see Section~\ref{sec:experiments} for results with the
AUC-specific protocol).
On this dataset, a fully trained logistic regression model yields scores
of positive and negative points that are well separated. Hence, they can
be quantized to a sufficiently small domain for the protocol to achieve small
error.

%% file: subfiles/accuracyproof_arxiv.tex
We are now ready to prove Theorem~\ref{thm:aucacc}.
\begin{proof}[Proof of Theorem \ref{thm:aucacc}]
   The estimation error $E_p=\widehat{\uauc}(
   \hat{h}^+_p,\hat{h}^-_p) - \uauc
(h_p^+,h_p^-)$ at a given
node $p$ can be written recursively as follows:
  \begin{equation*}
    E_p = \left\{
      \begin{array}{ll}
        \frac{1}{2}(\hat{h}^+_{p\cdot 1}+\hat{h}^+_{p\cdot 0})(\hat{h}^-_{p\cdot 1}+\hat{h}^-_{p\cdot 0})-\uauc(h^+_{p}, h^-_{p}) & \quad \mbox{if } p\in \discarded  \\
        0 & \quad \mbox{if } p \mbox{ is a leaf}\\
        \hat{h}^+_{p\cdot 1} \hat{h}^-_{p\cdot 0}-h^+_{p\cdot 1}h^-_{p\cdot 0}+E_{p\cdot 0}+E_{p\cdot 1}  & \quad  \mbox{if } p\in\recursed
      \end{array}
    \right.
  \end{equation*}

  We will consider the error $E_{\lambda}$ in two parts. Firstly, there is the contribution $E^\discarded$ from those prefixes $p\in \discarded$ which we define by setting
  \begin{equation*}
    E^{\discarded}_m=\sum_{p\in \discarded_{m-1}}\frac{1}{2}(\hat{h}^+_{p\cdot 1}+\hat{h}^+_{p\cdot 0})(\hat{h}^-_{p\cdot 1}+\hat{h}^-_{p\cdot 0})-\uauc(h^+_{p}, h^-_{p})
  \end{equation*}
  and $E^\discarded=\sum_{m\in [\alpha]}E^{\discarded}_m$. Secondly, there is the contribution from the prefixes $p\in \recursed$ excluding their recursive subcalls which we define by setting
 \begin{equation*}
   E^{\recursed}_m=\sum_{p\in \recursed_{m-1}} \hat{h}^+_{p\cdot 1} \hat{h}^-_{p\cdot 0}-h^+_{p\cdot 1}h^-_{p\cdot 0}
 \end{equation*}
 and $E^\recursed=\sum_{m\in [\alpha]}E^{\recursed}_m$. In bounding both of these we will make use of conditioning on $\mathcal{F}_{m}=(\hat{h}^-_p,\hat{h}^+_p)_{p\in \{0,1\}^{\leq m}}$ i.e. the answers of the frequency oracles for layers up to $m$.

 We start by bounding $E^\recursed$. For any $m\in [\alpha]$,
 we first show that $E^\recursed_m$ is a martingale difference sequence i.e.
 \begin{align*}
   \EE(E^\recursed_m|\mathcal{F}_{m-1})&=\EE(\sum_{p\in \recursed^{m-1}}(\hat{h}^+_{p\cdot 1} \hat{h}^-_{p\cdot 0}-h^+_{p\cdot 1}h^-_{p\cdot 0})|\mathcal{F}_{m-1}) \\
                                       &=\sum_{p\in \recursed^{m-1}}\EE((h^+_{p\cdot 1}+e^+_{p\cdot 1})(h^-_{p\cdot 0}+e^-_{p\cdot 0})-h^+_{p\cdot 1}h^-_{p\cdot 0}) \\
                                       &=\sum_{p\in \recursed^{m-1}}\EE(h^+_{p\cdot 1}e^-_{p\cdot 0}+e^+_{p\cdot 0}h^-_{p\cdot 0}+e^+_{p\cdot 1}e^-_{p\cdot 0}) \\
                                       &=0
 \end{align*}
 where the final equality holds because $\EE(e^\pm_p)=0$ for all $p$ and $e^+_{p\cdot 1}$ and $e^-_{p\cdot 0}$ are independent. From this we can conclude that for $m'>m$
 \begin{align*}
   \EE(E^\recursed_mE^\recursed_{m'})&=\EE(\EE(E^\recursed_mE^\recursed_{m'}|\mathcal{F}_{m'-1}))) 
                                     =\EE(E^\recursed_m\EE(E^\recursed_{m'}| 
                                     \mathcal{F}_{m'-1})) 
                                     =\EE(0)=0
 \end{align*}
 and thus
 \begin{align}
   \label{ERmidbound}
   \EE({E^\recursed}^2)&=\EE(\sum_{m\in [\alpha]}\sum_{m'\in [\alpha]}E^\recursed_mE^\recursed_{m'})
                       =\sum_{m\in [\alpha]}\EE({E^\recursed_m}^2)
                       = \sum_{m\in [\alpha]}\EE(\EE({E^\recursed_m}^2|\mathcal{F}_{m-1})).
 \end{align}
 
 Next we shall bound $\EE({E^\recursed_m}^2|\mathcal{F}_{m-1})$. We start by writing out
 \begin{equation*}
   \EE({E^\recursed_m}^2|\mathcal{F}_{m-1})= \EE((\sum_{p\in \recursed^{m-1}}(\hat{h}^+_{p\cdot 1} \hat{h}^-_{p\cdot 0}-h^+_{p\cdot 1}h^-_{p\cdot 0}))^2|\mathcal{F}_{m-1}).
 \end{equation*}
 By Equation \ref{ERmidbound} this becomes
 \begin{equation*}
   \EE({E^\recursed_m}^2|\mathcal{F}_{m-1})= \sum_{p\in \recursed^{m-1}}\EE((\hat{h}^+_{p\cdot 1} \hat{h}^-_{p\cdot 0}-h^+_{p\cdot 1}h^-_{p\cdot 0})^2|\mathcal{F}_{m-1}).
 \end{equation*}
 After expanding the above and removing all the terms that are zero, because they are the expected value of the product of $e^\pm_{p\cdot i}$ with something independent of it, we are left with
 \begin{align*}
   \EE({E^\recursed_m}^2|\mathcal{F}_{m-1})&=\sum_{p\in \recursed^{m-1}}\EE({h^+_{p\cdot 1}}^2{e^-_{p\cdot 0}}^2+{e^+_{p\cdot 0}}^2{h^-_{p\cdot 0}}^2+{e^+_{p\cdot 1}}^2{e^-_{p\cdot 0}}^2) \\
   &\leq \sum_{p\in \recursed^{m-1}}(v^-{h^+_{p\cdot 1}}^2+v^+{h^-_{p\cdot 0}}^2)+|\recursed^m|v^+v^- \\
   &\leq v^-{n^+}^2+v^+{n^-}^2+|\recursed^m|v^+v^-. 
 \end{align*}
 Subbing this into \eqref{ERmidbound} and using Lemma~\ref{lem:recursed} gives
 \begin{align*}
   \EE({E^\recursed}^2)&\leq \alpha \max_m \EE( v^-{n^+}^2+v^+
   {n^-}^2+|\recursed^m|v^+v^- ) \\
                       &= \alpha (v^-{n^+}^2+v^+{n^-}^2+\max_m\EE(|\recursed^m|)v^+v^-) \\
                       & \leq n^+n^-C\alpha^2(n^++n^-+ \frac{C\alpha}{\sqrt{a}-1}\sqrt{\frac{n}{2C\alpha}}) \\
                       &\leq n^+n^-C\alpha^2(n+\frac{\sqrt{C\alpha n}}{\sqrt{2}(\sqrt{a}-1)}) \\
                       &=:B^\recursed
 \end{align*}

  To bound $E^\discarded$, first define $E^F_m=\sum_{p\in \discarded_{m-1}}\frac{1}{2}(\hat{h}^+_{p\cdot 1}+\hat{h}^+_{p\cdot 0})(\hat{h}^-_{p\cdot 1}+\hat{h}^-_{p\cdot 0})-\frac{1}{2}h^+_ph^-_p$ and $E^G_m=\sum_{p\in \discarded_m} \frac{1}{2}h^+_ph^-_p-\uauc(h^+_{p}, h^-_{p})$.  We refer to the leaves in $[0..2^\alpha-1]$ {\em covered} by a path $p$ as $\interval(p) = \{i\in [0..d-1]: p\preceq b_i\}$. Now note that
  \begin{align*}
    E^{\discarded}&= \sum_{p\in \discarded} \frac{1}{2}(\hat{h}^+_{p\cdot 1}+\hat{h}^+_{p\cdot 0})(\hat{h}^-_{p\cdot 1}+\hat{h}^-_{p\cdot 0})-\sum_{i\in \interval(p)}h^+_{b_i}\sum_{j\in \interval(p),j<i}h^-_{b_j} 
                  = \sum_{m\in [\alpha]} E^F_m+ \sum_{m\in [\alpha]} E^G_m.
  \end{align*}

  We now bound $E^F_m$ and $E^G_m$ separately. For a leaf node $s$, let us denote by $v(s)$ the {\em unique} node in $D$ that is a prefix of $s$. We then have:
  \begin{align*}
    \EE((\sum_{m\in [\alpha]} E^G_m)^2)&=\EE((\sum_{p\in \discarded} \frac{1}{2}h^+_ph^-_p-\uauc(h^+_{p}, h^-_{p}))^2) \\
                                       &\leq \EE((\sum_{p\in \discarded} \frac{1}{2}h^+_ph^-_p)^2) 
                                       =\frac{1}{4}\EE((\sum_{p\in \discarded}\sum_{i\in \interval(p)}h^+_{b_i}\sum_{j\in \interval(p),j<i}h^-_{b_j})^2) \\
                                       &\leq \frac{{n^+}^2}{4}\max_{s\in [0..d-1]}\EE((h^-_{v(s)})^2).
  \end{align*}
  We can then bound
  \begin{align*}
    \EE({h^-_{v(s)}}^2)=&\sum_{p \preceq p(s)} \EE({h^-_{p}}^2\mathbb{I}_{p=v(s)}) 
    =\sum_{p \preceq p(s)} \EE((\hat{h}^-_{p}-e^-_{p})^2\mathbb{I}_{p=v(s)}) \\
    \leq&\sum_{p \preceq s} \EE((2\sqrt{av^-}-e^-_p)^2\mathbb{I}_{p=v(s)}) 
    \leq\sum_{p \preceq s} \EE(4av^--4\sqrt{av^-}e^-_p+{e^-_p}^2) \\
    \leq&(4a+1)\alpha v^-.
  \end{align*}
  Thus
  \begin{align*}
    \EE({E^G}^2)\leq& {n^+}^2(a+1/4)\alpha v^- 
    \leq C(a+1/4)n^-{n^+}^2\alpha^2.
  \end{align*}
  Furthermore by symmetry between $-$ and $+$
  \begin{equation*}
    \EE({E^G}^2)\leq C(a+1/4)n^-n^+\min(n^-,n^+)\alpha^2=:B^G.
  \end{equation*}

  Secondly we bound $\sum_{m\in [\alpha]} E^F_m$. Note that $E^F_{m-1}$ is a function of $\mathcal{F}_{m-1}$ and $\EE(E^F_m|\mathcal{F}_{m-1})=0$ so
  \begin{align*}
    \EE((\sum_mE^F_m)^2)&=\EE(\sum_m{E^F_m}^2) 
    =\sum_m\EE({E^F_m}^2) 
    =\sum_m\EE(\EE({E^F_m}^2|\mathcal{F}_{m-1})) \\
    &\leq\sum_m\EE(\EE((\sum_{p\in \discarded_{m-1}}\frac{1}{2}(\hat{h}^+_{p\cdot 1}+\hat{h}^+_{p\cdot 0})(\hat{h}^-_{p\cdot 1}+\hat{h}^-_{p\cdot 0})-\frac{1}{2}h^+_ph^-_p)^2|\mathcal{F}_{m-1})) \\
  \end{align*}
  Similarly to the bound on $E^\recursed$, we now apply the pairwise independence property and note that $\hat{h}^{\pm}_{p\cdot 1}+\hat{h}^{\pm}_{p\cdot 0}$ is an unbiased estimator of $h^{\pm}_p$ with variance bounded by $2v^{\pm}$. This results in
  \begin{align*}
    \EE((\sum_mE^F_m)^2)&\leq\sum_m\EE(\sum_{p\in \activated_{m-1}}\mathbb{I}_{\tilde{h}^+_p\tilde{h}^-_p<\tau}\EE({h^+_p}^2v^-/2+v^+{h^-_p}^2/2+v^+v^-|\mathcal{F}_{m-1})) \\
                        &\leq\sum_mv^-\EE(\sum_{p\in \activated_{m-1}}\mathbb{I}_{\tilde{h}^+_p\tilde{h}^-_p<\tau}{h^+_p}^2)/2+v^+\EE(\sum_{p\in \activated_{m-1}}\mathbb{I}_{\tilde{h}^+_p\tilde{h}^-_p<\tau}{h^-_p}^2)/2+v^+v^-\EE(|\activated_{m-1}|).
  \end{align*}
  Noting that $\EE(\mathbb{I}_{\tilde{h}^+_p\tilde{h}^-_p<\tau}{h^+_p}^2)\leq \min({h^+_p}^2,\frac{{h^+_p}^2v^+}{(h^+_p-\sqrt{v^+})^2})\leq 4v^+$ and that $\EE(|\activated_{m-1}|)= 2\EE(|\recursed_{m-2}|)\leq \frac{1}{\sqrt{a}-1}\sqrt{\frac{n}{2C\alpha}}$ gives
  \begin{align*}
    \EE((\sum_mE^F_m)^2)&\leq\sum_m E(|\activated_{m-1}|) 5 v^+v^- 
    \leq \frac{5\sqrt{2n} C^{1.5} \alpha^{2.5} n^+ n^-}{\sqrt{a}-1} 
    :=B^F.
  \end{align*}

  By the Cauchy-Schwarz inequality we can conclude that,
  \begin{align*}
    \EE({E^{\discarded}}^2)&\leq 2(B^G+B^F).
  \end{align*}

  Finally applying Cauchy-Schwarz again gives
  \begin{align*}
 \EE(E_\lambda^2)&=\EE((E^\recursed+E^G+E^F)^2) \\
                 &\leq 2B^\recursed+4B^G+4B^F \\
                 &=Cn^-n^+\alpha^2(2n+(4a+1)\min(n^-,n^+)+\frac{21\sqrt{2nC\alpha}}{\sqrt{a}-1})
  \end{align*}

\end{proof}

\begin{remark}
  The use of Cauchy-Schwarz to combine the separate errors in this proof is optimized for simplicity rather than minimizing the constants. At the expense of making the bound substantially more complicated a more precise analysis would reduce the bound. Gaining up to a factor of two in the case of very large $n$ and $min(n^-,n^+)$ small compared to $n$.
\end{remark}

The value of $a$ in Theorem~\ref{thm:aucacc} can be chosen to minimize the
error by taking
it to solve $
\sqrt{a}(\sqrt{a}-1)^2=21\sqrt{2nc\alpha}/(8\min(n^-,n^+))$ which is approximately $$a=(1+\sqrt{21/8}(2Cn\alpha/\min(n^-,n^+)^2)^{\frac{1}{4}})^2.$$
This leads to the following corollary.
\begin{corollary}
  Let $n_{\min}=\min(n^-,n^+)$ and $a=(1+\sqrt{21/8}(2Cn\alpha/n_{\min}^2)^{\frac{1}{4}})^2=1+o(1)$ then
  \begin{equation*}
    \mse(\aucest) \leq \frac{C}{n_{\min}(n-n_{\min})}\alpha^2(2n+(4a+1)n_{\min}+14(C n n_{\min}^2\alpha)^{\frac{1}{4}})=O(\alpha^2/n_{\min}).
  \end{equation*}
\end{corollary}

\begin{remark}
  For fixed $\alpha$ this is of the same order as the sampling error incurred in non-private AUC.
\end{remark}

\paragraph{Algorithm variant.} An alternative algorithm assigns a value of zero
to edges that it discards. For this algorithm a similar theorem holds by the same argument (actually a slightly simpler argument) the resulting error bound is
\begin{equation*}
  \mse(\uaucest) = Cn^-n^+\alpha^2(2n+(8a+2)\min(n^-,n^+)+\frac{\sqrt{2C\alpha n}}{(\sqrt{a}-1)}).
\end{equation*}
Note that the second term which is of leading order for $\min(n^-,n^+)$ a fixed fraction of $n$ is twice as large however the final term which is lower order is twenty-one times smaller. This lower order term might not be negligible in practice and so this algorithm should be considered. The corresponding choice of $a$ and bound on the final error is given by the following result.
\begin{corollary}
  Let $n_
    {\min}=\min(n^-,n^+)$ and $a=(1+\sqrt{\frac{\sqrt{2C\alpha n}}{16n_
    {\min}}})^2=1+o(1)$ then
  \begin{equation*}
    \mse(\aucest) \leq \frac{C}{n_
    {\min}(n-n_
    {\min})}\alpha^2((8a+2)n_
    {\min}+2n+(512C\alpha n n_
    {\min}^2)^{\frac{1}{4}})=O(\alpha^2/n_
    {\min})
  \end{equation*}
\end{corollary}